\documentclass{article}
\usepackage[T1]{fontenc}
\usepackage{amsmath}
\usepackage{amssymb}
\usepackage{mathtools}
\usepackage[numbers]{natbib}
\usepackage{bm}
\usepackage{dsfont}

\usepackage{color}
\usepackage{hyperref}
\usepackage{url}
\hypersetup{
    colorlinks,
    linkcolor={red!80!black},
    citecolor={blue!90!black},
    urlcolor={blue!90!black}
}

\usepackage[margin=1in]{geometry}
\usepackage{amsthm}
\usepackage{xcolor}

\usepackage[capitalize,noabbrev]{cleveref}
\crefname{equation}{Eq.}{Eqs.}
\crefname{figure}{Fig.}{Figs.}
\crefname{assumption}{Assumption}{Assumptions}

\title{Can We Theoretically Quantify the Impacts of Local Updates on the Generalization Performance of Federated Learning?}
\date{Auguest, 2024}

\author{%
    Peizhong Ju \\
    The Ohio State University \\
    \texttt{peizhong.ju@uky.edu} \\
    \and
    Haibo Yang \\
    Rochester Institute of Technology \\
    \texttt{hbycis@rit.edu} \\
    \and
    Jia Liu \\
    The Ohio State University \\
    \texttt{liu.1736@osu.edu} \\
    \and
    Yingbin Liang \\
    The Ohio State University \\
    \texttt{liang.889@osu.edu} \\
    \and
    Ness Shroff \\
    The Ohio State University \\
    \texttt{shroff.11@osu.edu}
}

\usepackage{amsmath}
\usepackage{mathtools}
\usepackage{bm}
\usepackage{dsfont}
\usepackage{tikz}
\usepackage{tikz-cd}
\usepackage[ruled, linesnumbered]{algorithm2e}
\usepackage{algpseudocode}
\usepackage{bbm}
\usepackage{tcolorbox}

\hypersetup{
    colorlinks,
    linkcolor={red!80!black},
    citecolor={blue!90!black},
    urlcolor={blue!90!black}
}

\usepackage{amsthm}
\usepackage{xcolor}
\newtheorem{theorem}{Theorem}
\newtheorem{lemma}{Lemma}
\newtheorem{proposition}{Proposition}
\newtheorem{assumption}{Assumption}

\usepackage[capitalize,noabbrev]{cleveref}
\crefname{equation}{Eq.}{Eqs.}
\crefname{assumption}{Assumption}{Assumptions}

\newcommand{\T}{^\top}
\newcommand{\defeq}{\coloneqq}
\newcommand{\norm}[1]{\left\|#1\right\|}

\DeclareMathOperator*{\argmin}{arg\,min}
\newcommand{\E}{\mathop{\mathbb{E}}}
\newcommand{\abs}[1]{\left|#1\right|}
\newcommand{\myMa}[1]{\begin{bmatrix}#1\end{bmatrix}}
\newcommand\numberthis{\addtocounter{equation}{1}\tag{\theequation}}
\DeclareMathOperator{\diag}{diag}

\newcommand{\nij}[2]{n_{(#1),#2}}
\newcommand{\sigmaij}[2]{\sigma_{(#1),#2}}

\newcommand{\XX}[2]{\mathbf{X}_{(#1),#2}}
\newcommand{\yy}[2]{\bm{y}_{(#1),#2}}
\newcommand{\ee}[2]{\bm{\epsilon}_{(#1),#2}}

\newcommand{\vw}{\bm{w}}
\newcommand{\vy}{\bm{y}}
\newcommand{\mX}{\mathbf{X}}
\newcommand{\what}{\hat{\vw}}
\newcommand{\wInit}{\hat{\vw}_0}
\newcommand{\vwTrue}{\vw^*}
\newcommand{\wijTrue}[2]{\vw_{(#1),#2}}
\newcommand{\wij}[2]{\hat{\vw}_{(#1),#2}^{K=\infty}}
\newcommand{\wavg}[1]{\hat{\vw}_{\text{avg},#1}^{K=\infty}}
\newcommand{\Pij}[2]{\mathbf{P}_{(#1),#2}}
\newcommand{\deltawt}[1]{{\bm{\Delta}_{#1}^{K=\infty}}}

\newcommand{\gammait}[2]{\bm{\gamma}_{(#1),#2}}

\newcommand{\wavgGeneral}[1]{\hat{\vw}_{\text{avg},#1}}
\newcommand{\wijGeneral}[2]{\hat{\vw}_{(#1),#2}}
\newcommand{\deltatGeneral}[1]{\bm{\Delta}_{#1}}
\newcommand{\gammaSquareAvg}{\overline{\norm{\bm{\gamma}}^2}}

\newcommand{\wwij}[2]{\hat{\vw}_{(#1),#2}^{K=1}}
\newcommand{\deltawwt}[1]{{\bm{\Delta}_{#1}^{K=1}}}
\newcommand{\wwavg}[1]{\hat{\vw}_{\text{avg},#1}^{K=1}}

\newcommand{\deltat}[1]{{\bm{\Delta}_{#1}^{K<\infty}}}
\newcommand{\wijj}[2]{\hat{\vw}_{(#1),#2}}
\newcommand{\wk}[1]{\wijj{i}{t,#1}}
\newcommand{\Xk}{\XX{i}{t,k}}
\newcommand{\yk}{\yy{i}{t,k}}
\newcommand{\ek}{\ee{i}{t,k}}
\newcommand{\wavgg}[1]{\hat{\vw}_{\text{avg},#1}^{K<\infty}}
\newcommand{\stepsizeit}[2]{\alpha_{(#1),#2}}
\newcommand{\step}{\stepsizeit{i}{t}}
\newcommand{\stepj}{\stepsizeit{j}{t}}
\newcommand{\nn}{\tilde{n}_{(i),t}}

\newcommand{\mySum}{\frac{1}{\sum_{i\in [m]}\nij{i}{t}}\sum_{i\in [m]}}
\newcommand{\mySumP}[1]{\frac{\sum_{i\in [m]}#1}{\sum_{i\in [m]}\nij{i}{t}}}
\newcommand{\mySumSquare}{\frac{1}{(\sum_{i\in [m]}\nij{i}{t})^2}\sum_{i\in [m]}}
\newcommand{\sumNeq}{\sum_{j\in [m]\setminus\{i\}}}
\newcommand{\mySumSquareP}[1]{\frac{#1}{(\sum_{i\in [m]}\nij{i}{t})^2}}
\newcommand{\mySumSquarePP}[1]{\frac{\sum_{i\in [m]}#1}{(\sum_{i\in [m]}\nij{i}{t})^2}}

\newcommand{\iMatrix}[1]{\mathbf{I}_{#1}}

\newcommand{\F}{\mathcal{F}}
\newcommand{\seq}[2]{\text{seq}_{#1}\left(#2\right)}
\renewcommand{\citet}{\cite}

\begin{document}

\maketitle

\renewcommand{\thefootnote}{\fnsymbol{footnote}} 
\footnotetext[1]{This work is published in MobiHoc 2024.} 

\begin{abstract}

Federated Learning (FL) has gained significant popularity due to its effectiveness in training machine learning models across diverse sites without requiring direct data sharing.
While various algorithms along with their optimization analyses have shown that FL with local updates is a communication-efficient distributed learning framework, the generalization performance of FL with local updates has received comparatively less attention.
This lack of investigation can be attributed to the complex interplay between data heterogeneity and infrequent communication due to the local updates within the FL framework.
This motivates us to investigate a fundamental question in FL: {\em Can we quantify the impact of data heterogeneity and local updates on the generalization performance for FL as the learning process evolves?}
To this end, we conduct a comprehensive theoretical study of FL's generalization performance using a linear model as the first step, where the data heterogeneity is considered for both the stationary and online/non-stationary cases. 
By providing closed-form expressions of the model error, we rigorously quantify the impact of the number of the local updates (denoted as $K$) under three settings ($K=1$, $K<\infty$, and $K=\infty$) and show how the generalization performance evolves with the number of rounds $t$. 
Our investigation also provides a comprehensive understanding of how different configurations (including the number of model parameters $p$ and the number of training samples $n$) contribute to the overall generalization performance, thus shedding new insights (such as benign overfitting) for implementing FL over networks.

\end{abstract}
\section{Introduction}
Federated Learning (FL) has recently emerged as a prominent paradigm in the realm of distributed learning, facilitating the collaborative training of machine learning models among clients under the orchestration of a central server. 
By offering privacy preservation, scalability, and collaborative intelligence, FL holds great potential to revolutionize industries in healthcare, finance, IoT, among others~\cite{yang2019flconcept,xu2021federated,long2020federated,khan2021federated}.
In FL, the federated averaging (FedAvg) algorithm~\cite{mcmahan2017communication} and its variants have become the prevailing approach.
FedAvg leverages local computation at each client and employs a centralized parameter server to aggregate and update the
model parameters. 
The unique feature of FedAvg is that each client runs {\em multiple local stochastic gradient descent (SGD) steps} between two consecutive communication rounds to reduce the communication frequency between the clients and server.
In the literature, it has been shown that FedAvg-type algorithms with local updates achieve fast convergence rates while enjoying a low communication complexity.
More importantly, the low communication complexity due to local SGD updates renders FedAvg-type algorithms ideal for deployment over wireless edge networks, where the communications links could likely be highly dynamic, stochastic, and unreliable.

However, even with the evident benefit of being communication-efficient, the impact of local updates on the {\bf generalization performance} of FedAvg-type algorithms remains poorly understood.
The lack of such theoretical understanding affects the long-term and large-scale adoption of FL. 
%
Particularly, in the FL literature, there remains a significant amount of controversy over how the FL generalization performance is affected under the intricate interplay between {\em data heterogeneity} and {\em local update steps}.
Specifically, some researchers speculated that data heterogeneity results in poor generalization through empirical experiments~\cite{caldarola2022improving,zhao2018niid}, while other works argued that FedAvg can generalize very well with data heterogeneity~\cite{wang2022unreasonable,lin2019don,wang2021cooperative,ortiz2021trade}.
Notably, it has been empirically demonstrated that FedAvg-type algorithms using a fine-tuned number of local update steps exhibit a better generalization performance than the parallel stochastic gradient descent (SGD) algorithm~\cite{lin2019don,wang2021cooperative,ortiz2021trade}.
So far, however, there is {\em no} theoretical guiding principle on how to choose an appropriate number of local update steps to achieve good generalization performance in the FL literature.
Given the ever-increasing importance of FL as a distributed learning mechanism over networks, a compelling open question arises: 

\begin{tcolorbox}[left=1.2pt,right=1.2pt,top=1.2pt,bottom=1.2pt]
\textbf{(Q)}: How does the local update process, when coupled with data heterogeneity, impact the generalization performance of federated learning?
\end{tcolorbox}

In the FL literature, there have been some initial attempts to theoretically understand the generalization performance of FL (see Section~\ref{app.more_related_work} for more discussions).
The first line of work employs the traditional analytical tools from statistical learning, such as the ``probably approximately correct'' (PAC) framework. 
These works focus on the domain changes due to the data and system heterogeneity. 
For example, the works in \citet{yuan2022what} and \citet{hu2023generalization} assumed that clients’ data distributions are drawn from a meta-population distribution. Accordingly, two generalization gaps in FL are defined. One is the participation generalization gap, which measures the difference between the empirical and expected risk for participating clients; and the other is the non-participation generalization gap, which measures the difference in the expected risk between participating and non-participating clients.
The second class of works studied the training dynamic near a manifold of minima and focused on the effect of stochastic gradient noise on generalization.
For instance, the FL generalization behavior was investigated in \citet{caldarola2022improving} through the lens of the geometry of the loss and Hessian eigenspectrum, while the long-term FL generlization behavior was studied in 
\citet{gu2022and} using the stochastic differential equation (SDE) approximation.
Recently, researchers studied FL generalization under data heterogeneity through algorithmic stability \citet{sun2023understanding}.
Also, rate-distortion theoretic bounds on FL
the generalization have been established in \citet{sefidgaran2023federated}.
%

Despite the valuable insights on FL generalization offered by the aforementioned existing works, it is important to note that they primarily yield asymptotic results by focusing on domain changes or describing asymptotic behavior such as sufficiently large communication rounds and fine-tuned local steps. 
Hence, these works all fell short of providing an explicit relationship to characterize how critical factors in FL, (e.g., the number of local updates, the number of communication rounds, and data heterogeneity) affect the generalization of FL in general. 
%
To bridge this gap, as a starting point, we conduct the first theoretical study on the number of local updates on FL's generalization performance based on the recent double-descent theoretical framework for over-parameterized learning models.
Our objective is {\em to explicitly quantify the influence of local update steps, data heterogeneity, and the total number of communication rounds on the generalization performance of FL,} all of which are particularly relevant to the deployment of FL over edge networks.
We highlight our contributions as follows:
\begin{list}{\labelitemi}{\leftmargin=1em \itemindent=0em \itemsep=.2em}
\item To lay a theoretical foundation for FL generalization, we start with a linear model with Gaussian features in over-parameterized (related to benign overfitting \citep{li2023estimation,ju2020overfitting,belkin2020two}) and under-parameterized regimes. 
Specifically, in round $t$ of FL, agent $i$ aims to learn a model $\bm{w}$ through its own local data that follow the underlying ground truth model $\yy{i}{t} = \XX{i}{t}\T \wijTrue{i}{t} + \ee{i}{t}, i \in [m]$, where $\wijTrue{i}{t}$ is the ground-truth slope. 
By considering different $\wijTrue{i}{t}$, the data samples ($\mathbf{X}_{i}, \mathbf{y}_i$) can simulate various patterns of data heterogeneity, including both stationary (i.e., $\wijTrue{i}{t} = \bm{w}_{(i)}$) and online/non-stationary (i.e., $\wijTrue{i}{t}$ is time-varying) cases.
Utilizing this model allows us to efficiently disentangle the distinct influences of heterogeneous data, local update processes, and communication rounds in FL.

\item Based on the aforementioned analytical model, we provide {\em closed-form} expressions of the generalization error of FedAvg-type algorithms in terms of the number of local update steps. 
Specifically, we rigorously quantify the impact of local update steps (denoted as $K$) under three representative regimes ($K=1$, $K< \infty$, and $K=\infty$) and show how the generalization performance evolves with respect to the number of communication rounds $t$. 
Our results reveal some interesting insights: 1) a good pre-trained model ``helps'' but only to some extent; 2) the effect of noise and heterogeneity accumulates but can be limited; 3) the optimal number of local updates exists only in ``some cases,'' hence {\em resolving the empirical controversy} regarding the effect of $K$.

\item We note that, in addition to offering insights into FL's deployment over edge networks, our work is also of independent interest in learning theory. 
Specifically, our closed-form expressions of the FL generalization error contribute to answering, in the FL context, the fundamental question of why an over-parameterized model can generalize well. 
Note that over-parameterized deep neural networks (DNNs) have been widely used in machine learning (including FL), although it remains a myth why they can generalize well (also known as ``benign overfitting''). 
In the recent literature, a promising approach toward resolving the benign overfitting question is the so-called ``double-descent'' theoretical framework \citep{belkin2018understand, belkin2019two,bartlett2019benign,hastie2019surprises,muthukumar2019harmless,mitra2019understanding,ju2020overfitting} that starts from over-parameterized linear models. 
In this work, we extend such double-descent analysis into the FL regime where the distributed learning procedure is more complex than classical centralized learning due to the complications of local updates and data heterogeneity.

\end{list}

The rest of this paper is organized as follows.
\cref{app.more_related_work} reviews the literature to put our work in
comparative perspectives. In \cref{sec:systemModel}, we introduce the over-parameterized linear model in our FL system.
\cref{sec:mainResult} presents the main generalization analysis, which is followed by the (sketched) proofs of some key results in \cref{app.proof_opt_K,sec.proof_sketch}. The conclusion is in \cref{sec:conclusion}.
\section{Related Work}\label{app.more_related_work}

\textbf{1) Federated Learning:}
Federated Learning (FL) has emerged as a popular distributed learning framework, which harnesses the collaborative power of multiple clients to learn a shared model~\citep{li2019federated, yang2019federated, kairouz2019advances}. 
Since its inception, FL systems have demonstrated increasing prowess, effectively handling diverse forms of heterogeneity in data, network environments, and worker computing capabilities.
A large number of FL algorithms, including FedAvg~\citep{mcmahan2016communication} and its various adaptations~\citep{li2018fedprox, zhang2020fedpd, karimireddy2020scaffold, karimireddy2020mime, acar2021feddyn, yang2021achieving, yang2022anarchic}, have been proposed in the literature. 
However, it is worth noting that these works only provide insights into the convergence in optimization, while lacking the understanding of generalization performance for FL.

\smallskip
\noindent \textbf{2) Generalization Performance of FL:}
In the literature, there have been relatively limited studies on the generalization of FL. 
We categorize these works into three distinct classes.
The first line of work employs the traditional analytical tools from statistical learning.
The work in \citet{yuan2022what} assumed that clients’ data distributions are drawn from a meta-population distribution. Accordingly, they define two generalization gaps in FL: one is the participation generalization gap to measure the difference between the empirical and expected risk for participating clients, the same as the definition in classic statistical learning; the second is the non-participation generalization gap, which measures the difference of the expected risk between participating and non-participating clients.
Following this two-level distribution framework, sharper bounds are provided ~\cite{hu2023generalization}.
Also, the probably approximately correct (PAC) Bayesian framework is used in \citet{zhao2023federated} to investigate a tailored generalization bound for heterogeneous data in FL.
Recently, some researchers studied FL generalization under data heterogeneity through algorithmic stability \citet{sun2023understanding}.
Meanwhile, PAC-Bayes and rate-distortion theoretic bounds on
FL generalization errors have been established in \citet{sefidgaran2023federated}.
Similar tools are also used to study FL generalization in \citep{chor2023more,barnes2022improved,sefidgaran2022rate,huang2021fl}.

The second line of work studied the FL training dynamic near a manifold of minima and focused on the effect of stochastic gradient noise on generalization. 
These works used ``sharpness'' as a tool for characterizing generalization. 
For instance, the generalization behavior 
was investigated in \citet{caldarola2022improving} and \citet{shi2023make} through the lens of the geometry of the loss and Hessian eigenspectrum, which links the model’s lack of generalization capacity to the sharpness of the solution under ideal client participation.
Based on sharpness, a momentum algorithm with better generalization was proposed in \citet{qu2022generalized}.
Also, the long-term generalization behavior of FL is studied in \citet{gu2022and} using the stochastic differential equation (SDE) approximation, which showed that local steps could lead to better generalization under appropriate conditions (e.g., a sufficiently small learning rate, a sufficiently large number of communication rounds, and an appropriately chosen number of local update steps).

We note that all of these existing works on FL generalization only provide asymptotic results on domain changes or describe limiting behavior, such as a large number of communication rounds under a carefully chosen number of local updates. 
Consequently, they all fell short of establishing a direct quantification that demonstrates how key FL factors (i.e., data heterogeneity, the number of local updates, and the communication round) affect FL generalization.

\smallskip
\noindent
\textbf{3) Benign Overfitting and Double Descent:} 
Since our work is intimately related to the double-descent framework for resolving the ``benign overfitting'' mystery, it is also insightful to provide a quick overview of this research area here.
As an initial step to understanding why over-parameterized DNNs generalize well (i.e., ``benign overfitting'') and exhibit the so-called ``double-descent'' phenomenon (i.e., the generalization risk descends again beyond the conventional ``U-shape'' curve in the over-parameterized regime), early attempts in this area started from exploring the minimum $\ell_2$-norm \citep{belkin2018understand, belkin2019two,bartlett2019benign,hastie2019surprises,muthukumar2019harmless} or $\ell_1$-norm  \citep{mitra2019understanding,ju2020overfitting} overfitted solutions of the linear models with Gaussian or Fourier features.
Later studies in this area investigated the generalization performance of overfitted solutions of shallow neural network approximations. 
For example, researchers have considered random feature (RF) models \citep{mei2019generalization}, two-layer neural tangent kernel (NTK) models \citep{arora2019fine,satpathi2021dynamics,ju2021generalization}, and three-layer NTK models \citep{ju2022generalization}. 
Note that all of these studies have focused only on the centralized learning settings, while our work considers the benign overfitting phenomenon in the FL settings, which are far more complex due to the multi-agent nature and unique complications due to FL, such as local updates and data heterogeneity.

\section{System Model} \label{sec:systemModel}

\subsection{The Ground-Truth Model, the Learning Model, and Training Samples}

As a first step toward a theoretical understanding of the impacts of local updates on the FL generalization performance, we consider the general linear ground truth model which is widely used in the literature on machine learning theory (e.g., \citet{li2023estimation,ju2020overfitting,belkin2020two}):
\begin{align}
    y = \tilde{\bm{x}}\T \tilde{\bm{w}}+\epsilon,\label{eq.ground_truth_sparse}
\end{align}
where $\tilde{\bm{x}} \in \mathds{R}^s$ denotes the feature vector that consists of $s$ true features, $\tilde{\bm{w}}\in \mathds{R}^s$ denotes the corresponding ground-truth model parameters, and $\epsilon \in \mathds{R}$ denote the noise in the output $y\in \mathds{R}$.

Let $p$ denote the number of features/parameters for the chosen learning model. In other words, a sample is in the form of $(\bm{x}\in \mathds{R}^p,y)$. 
In practice, the number of features could be large (may or may not be necessary) to make sure that all true features are included. 
Thus, we assume that $p\geq s$ and those $p$ features include all necessary features\footnote{Our result can be generalized to the case of missing features by treating the missing part as noise.}. 
Without loss of generality, we let $\tilde{\bm{x}}$ be the first $s$ elements of $\bm{x}$. Correspondingly, we define $\bm{w} \defeq  \left[\begin{smallmatrix}\tilde{\bm{w}}\\\bm{0}\end{smallmatrix}\right]\in \mathds{R}^p$. Thus, \cref{eq.ground_truth_sparse} can be rewritten as $y = \bm{x}\T \bm{w} + \epsilon$. We note that such a linear model is considered in many works on theoretical understanding of the
double-descent phenomenon in deep learning theory \cite{belkin2018understand, belkin2019two,bartlett2019benign,hastie2019surprises,muthukumar2019harmless,mitra2019understanding,ju2020overfitting}.  In \cref{sec:mainResult}, we will also show that these linear models lead to insights that have been observed in practical (non-linear) FL.

Consider the FL setting with $m$ clients, where the communication rounds are indexed by $t=1,2,\cdots,T$. We use $[m]$ to denote the set $\{1,2,\cdots,m\}$, and use $[T]$ to denote the set $\{1,2,\cdots,T\}$. 
We use the subscript $(\cdot)_{(i),t}$ to denote a quantity for the $i$-th agent at the $t$-th round.
In the $t$-th communication round of FL, the $i$-th client uses $\nij{i}{t}$ training samples.
Stacking these training samples, we have the following matrix equation.
\begin{align}\label{eq.linear_model}
    &\yy{i}{t} = \XX{i}{t}\T \wijTrue{i}{t} + \ee{i}{t},
\end{align}
where $\XX{i}{t}\in \mathds{R}^{p \times \nij{i}{t}}$, $\wijTrue{i}{t}\in \mathds{R}^p$, $\yy{i}{t}\in \mathds{R}^{\nij{i}{t}}$, and $\ee{i}{t}\in \mathds{R}^{\nij{i}{t}}$. 
It is worth noting that Eq.~\ref{eq.linear_model} is quite general, including both stationary scenarios where $\wijTrue{i}{t} = \mathbf{w}_{i}$ and non-stationary scenarios with time-varying $\wijTrue{i}{t}$ that accounts for environmental changes at the edge devices.
The subscript notation $(\cdot)_{(i),t}$ in $\wijTrue{i}{t}$ offers a more general framework to model various complications in FL, such as unbalanced data, heterogeneity, and non-stationarity.
In general FL, there exist ground-truth parameters $\vwTrue\in \mathds{R}^p$ in the system, which corresponds to the target solution of FL.
For example, in simple FL with balanced data, the ground truth is can be written as $\vwTrue = \frac{1}{m} \sum_{i \in [m]} \mathbf{w}_{i}$.


\subsection{Data Distribution, Heterogeneity, and Non-stationarity}

To analytically characterize the impact of local updates on the FL generalization performance, we need some assumptions on the distribution of the training data $\left(\XX{i}{t},\ \yy{i}{t}\right)_{i\in [m],t=1,2,\cdots,T}$.
First, we adopt the independent Gaussian features and noise assumption, which is a common assumption in the literature (e.g., \citep{belkin2018understand,ju2020overfitting}) for analyzing over-parameterized generalization performance. Specifically, we have the following assumption:

\begin{assumption}\label[assumption]{as.Gaussian}
For any $i,t$, each element of $\XX{i}{t}$ follows i.i.d. standard Gaussian distribution, and each element of $\ee{i}{t}$ follows independent Gaussian distribution with zero mean and variance $\sigmaij{i}{t}^2$.
\end{assumption}

\cref{as.Gaussian} assumes that each dataset per round is unique and freshly obtained, mirroring the conditions of an online data acquisition environment. Besides, it also serves as a realistic approximation for scenarios involving large, fixed datasets.

Since we consider linear models, the heterogeneity of the variance of $\XX{i}{t}$ can be normalized, i.e., it is equivalent to only consider the heterogeneity of the variance of $\ee{i}{t}$ as described in \cref{as.Gaussian}.
Note that although $\XX{i}{t}$ has identical distribution among different clients, the training data are heterogeneous in $\yy{i}{t}$ because $\wijTrue{i}{t}$ can be different and $\sigmaij{i}{t}$ may have different values. In other words, $\yy{i}{t}$ and $\yy{j}{t}$ may have different distributions for different $i$ and $j$ in our model. 
To quantify the level of heterogeneity in the ground-truth $\wijTrue{i}{t}$, we define 
\begin{align}\label{eq.define_gammait}
    \gammait{i}{t}\defeq \vwTrue-\wijTrue{i}{t}.
\end{align}
Intuitively, $\gammait{i}{t}$ describes the (small) perturbation of agent $i$'s ground truth at the $t$-th round with respect to the target ground truth $\vwTrue$.
The quantification of data heterogeneity here aligns with established research in FL, where the assumption $\| \nabla f_i(\mathbf{x}) - \nabla f(\mathbf{x}) \|^2 \leq \sigma_G^2$ is commonly used to quantify data heterogeneity~\cite{mcmahan2021fl}. 
In the case of a linear model, this assumption can be equivalently expressed as $\| \gammait{i}{t} \|^2 \leq \sigma_G^2$.




\subsection{Federated Learning Process}

We use mean-squared-error (MSE) as the training loss, i.e., the training loss of the parameters $\what$ on $n$ samples $(\mathbf{X},\bm{y})$ is defined as:
\begin{align}\label{eq.train_loss}
    L(\what;\mathbf{X},\bm{y}) \defeq \frac{1}{2n} \norm{\bm{y}-\mathbf{X}\T \what}^2.
\end{align}
We consider the FedAvg algorithm~\cite{mcmahan2017communication}, where a central server averages the local updates of each agent (weighted by each agent's number of samples) and then distributes the weighted averaged result to all agents as the initial point of the next local update. We use $\wavgGeneral{t} \in \mathds{R}^p$ to denote the weighted average result at round $t$, and use $\wijGeneral{i}{t}\in \mathds{R}^p$ to denote the result of the local update of agent $i$ at round $t$. The weighted average can be expressed as:
\begin{align}\label{eq.FedAvg}
    \wavgGeneral{t}\defeq \mySumP{\nij{i}{t} \wijGeneral{i}{t}}.
\end{align}
Let $\wInit$ denote the initialization of the parameters (e.g., starting from a pre-trained model). 
For notational convenience, we define $\wavgGeneral{0}\defeq \wInit$.

Recall that the focus of this paper is to examine the impact of local updates on FL generalization. 
To this end, we use a parameter $K>0$ to denote the number of local update steps.
We consider the following three regimes in terms of different $K$ values: $K=1$, $K<\infty$, and $K=\infty$. 
We use superscripts $(\cdot)^{K=1}$, $(\cdot)^{K<\infty}$, and $(\cdot)^{K=\infty}$ to these cases, respectively. 
For example, $\wwavg{t}$ and $\wwij{i}{t}$ denote the values of $\wavgGeneral{t}$ and $\wijGeneral{i}{t}$, respectively, when we consider the setting of $K=1$.

\subsubsection{\texorpdfstring{$K=1$}{K=1} (One-Step Gradient)} \label{subsec.model_K1}

The simplest algorithm in FL is to perform only one gradient step in each client's local update. 
Specifically, for all clients $i\in [m]$ and each round $t=1,2,\cdots,T$, the result of the local step (denoted by $\wwij{i}{t}$) can be written as:
\begin{align*}
    &\wwij{i}{t}\defeq \wwavg{t-1}-\step \frac{\partial L(\wwavg{t-1};\XX{i}{t},\yy{i}{t})}{\partial \wwavg{t-1}},
\end{align*}
where $\step>0$ denotes client $i$'s learning rate (i.e., step size) of the local update in round $t$.

\subsubsection{General \texorpdfstring{$K<\infty$}{finite K} (Multi-Batch Local Updates)}\label{subsubsec.finite_K}

The general case in FL is that in each round $t$, every client performs local updates multiple (finite) times. 
In the $k$-th update, client $i$ uses $\nn$ data $(\Xk,\yk)$ (as a batch) where $\Xk\in \mathds{R}^{p\times \nn}$ and $\yk\in \mathds{R}^{\nn}$. 
In this paper, we consider the case where $\Xk$ for all $k\in [K]$ are disjoint and their union is $\XX{i}{t}$. 
In other words, the data $\XX{i}{t}$ are partitioned evenly into $K$ batches (and thus we have $K \cdot \nn = \nij{i}{t}$).
We define $\wk{k}$ as the result after the $k$-th batch for client $i$ in round $t$. 
Specifically, for the local update in the $k$-th batch ($k=1,2,\cdots,K$), we have
\begin{align*}
    \wk{k} \!&\defeq\! \wk{k-1} \!\!-\!\! \step \frac{\partial L(\wk{k-1};\Xk,\yk)}{\partial \wk{k-1}},
\end{align*}
where $\step > 0$ denotes the learning rate.
We note that $\wk{0}\defeq \wwavg{t-1}$ and $\wijj{i}{t}\defeq \wk{K}$. 
Also, the general case degenerates to that of \cref{subsec.model_K1} when $K=1$.

\subsubsection{\texorpdfstring{$K=\infty$}{K is infinity} (Convergence in Local Update)}
In this case with $K=\infty$, we consider each client's solution that the local GD/SGD converges to\footnote{The difference between a very large but finite $K$-value and $K=\infty$ has been characterized in the literature of the convergence analysis on gradient descent, e.g., \citet{gower2018convergence,garrigos2023handbook}.}, which is different from \cref{subsec.model_K1,subsubsec.finite_K} where every sample is only trained once.
In the under-parameterized regime $p<\nij{i}{t}$, the convergence point at each client corresponds to the solution that minimizes the local training loss, i.e.,
\begin{align*}
    \wij{i}{t}\defeq \argmin_{\what} L(\what;\XX{i}{t},\yy{i}{t}),\ \text{ when }p < \nij{i}{t}.
\end{align*}
In the over-parameterized regime $p>\nij{i}{t}$, 
there are infinitely many solutions that make the training loss zero with probability $1$, i.e., overfitted solutions. 
It is known in the literature that an overfitted solution corresponding to GD/SGD on a linear model in the over-parameterized regime has the smallest $\ell_2$-norm of the change of parameters \cite{gunasekar2018characterizing,lin2023theory}. Specifically, the convergence point of the local updates corresponds to the solution to the following optimization problem: for $t=1,2,\cdots,T$, when $p > \nij{i}{t}$, we have
\begin{align}
    &\wij{i}{t}\defeq \argmin_{\what}~~ \norm{\hat{\vw}-\wavg{t-1}}, \\
    &\text{subject to}~~ \XX{i}{t}\T\what=\yy{i}{t}.\label{eq:op}
\end{align}
The constraint in \cref{eq:op} implies that the training loss is exactly zero (i.e., overfitted, which is also known as the  interpolation regime).

\subsection{Generalization Performance Metric}\label{subsec.generalization_metric}

We then use the distance between the trained model $\hat{\vw}$ and the ground truth model $\vwTrue$, i.e., model error, to characterize the generalization performance: $L^{\text{model}}(\hat{\vw})=\norm{\hat{\vw}-\vwTrue}^2$. 
Such model error is equal to the expected test
error in some cases.\footnote{We can show that the model error is equal to the expected test error for noise-free data. See \cref{le.model_error}.}
For convenience, we define
\begin{align}
    \deltatGeneral{t}\defeq \vwTrue - \wavgGeneral{t},\qquad t=0,1,2,\cdots,T.\label{eq.def_delta}
\end{align}
Therefore, to characterize the generalization performance of FL at the end of round $t$, we need to quantify $\norm{\deltatGeneral{t}}^2$ with respect to $p$, $K$, $n$, learning rates, initialization, etc. Note that $\deltatGeneral{0}$ characterizes the difference between the initial weights $\wInit$ (which can be viewed as starting from an initial or pre-trained model) and the ideal solution $\vwTrue$ (thus $\deltatGeneral{0}$ is irrelevant to the configuration of $K$).

\subsection{Extra Notations}

Let $\seq{i}{\cdot}$ denote a sequence of numbers/vectors indexed by $i$.
For $l=1,2,\cdots$, and for a real number/vector $\bm{\beta}_0$, we define a mapping $\F$ as follows:
\begin{align}\label{eq.def_F_formula}
    \F(l, \bm{\beta}_0, \seq{i}{a_i}, \seq{i}{b_i}) \defeq \prod_{i=1}^l a_i \bm{\beta}_0 + \sum_{i=1}^l b_i \cdot \prod_{j=i+1}^l a_j.
\end{align}
\cref{eq.def_F_formula} corresponds to the general-term formula of $\bm{\beta}_l$ for the recurrence relation $\bm{\beta}_i = a_i \bm{\beta}_{i-1} + b_i$.

\section{Main Results} \label{sec:mainResult}

In this section, we will present the closed-form expression of $\E\norm{\deltatGeneral{t}}^2$ for all three cases of $K$-values. 
These expressions are complex since our system model considers both the non-stationarity along different rounds and the heterogeneity across different clients. 
To make our results more accessible, we also provide a simplified version of our results for the special case, where the system is stationary across rounds and the heterogeneity across clients are bounded. 
Specifically, the simple case is defined as: 
for all $i\in [m], t\in [T]$,
\begin{align}
    &\nij{i}{t}= n,\quad \step = \alpha,\quad \sigmaij{i}{t}=\sigma, \label{eq.simple_case1}\\
    & \sum\nolimits_{j\in [m]}\gammait{j}{t}= 0, \label{eq.simple_case4} \\
    &\frac{\sum_{j \in [m]}\norm{\gammait{j}{t}}^2}{m}= \gammaSquareAvg,\label{eq.simple_case5}
\end{align}
where 
$\gammaSquareAvg \geq 0$ denotes the level of heterogeneity. 
Here, we consider the balanced data case with a constant learning rate and constant noise in data.
The expression $\sum_{j\in [m]}\gammait{j}{t}= 0$ in Eq.~\eqref{eq.simple_case4} indicates that the ground-truth solution $\vwTrue$ is the average of the all clients' ground truth $\wijTrue{i}{t}$, i.e., $\vwTrue = \frac{1}{mT} \sum_{i \in [m], t \in [T]} \wijTrue{i}{t}$.
With the above notations, we are now ready to present our main results in the following subsections. It is important to note that our general results, including \cref{eq.result_K1,eq.multi_K,eq.inf_K_main,eq.underparameterized}, are derived independently of the more restrictive \cref{eq.simple_case1,eq.simple_case4,eq.simple_case5}, which are only applied in simplified scenarios such as \cref{eq.simple_K1,eq.simple_Kc,eq.main_thm_simple}.

\subsection{The \texorpdfstring{$K=1$}{K=1} Case}

We define the following short-hand notations:
\begin{align*}
    &\bm{g}_l^{K=1} \defeq  \F(l, \deltatGeneral{0},\seq{t}{\mySumP{\nij{i}{t}(1-\step)}}, \\
    & \qquad \qquad \seq{t}{\mySumP{\step\nij{i}{t}\gammait{i}{t}}}),\numberthis \label{eq.def_gt}\\
    &H_t \!\! \defeq \!\! \mySumSquareP{\left(\sum_{i\in [m]}\nij{i}{t}(1\!-\!\step)\right)^2 \!\!\!\!+\! \sum_{i\in [m]}\step^2 \nij{i}{t}(p\!\!+\!\!1) }, \numberthis \label{eq.def_Ht}\\
\end{align*}
\begin{align*}
    &G_t \!\! \defeq \!\! \mySumSquareP{\sum_{i\in [m]}\!\step^2 p \nij{i}{t}\sigmaij{i}{t}^2}\!\!+\!\!\mySumSquareP{\norm{\sum_{i\in [m]}\!\step \nij{i}{t}\gammait{i}{t}}^2} \\
    &+ \mySumSquareP{\sum_{i\in [m]} \step^2 \nij{i}{t}(p+1) \norm{\gammait{i}{t}}^2} + \\
&\mySumSquareP{2\left(\sum_{i\in [m]}\!\nij{i}{t}(1\!-\!\step)\right) \!\! \left(\sum_{i\in [m]}\!\nij{i}{t}\step \gammait{i}{t}\T\bm{g}_{t-1}\right)}\\
    &-\mySumSquareP{2\sum_{i\in [m]}\step^2 \nij{i}{t}(p+1) \gammait{i}{t}\T \bm{g}_{t-1}^{K=1}}.\numberthis \label{eq.def_Gt}
\end{align*}

\begin{theorem}\label{th.K1}
When $K=1$, we have
\begin{align}
    \E \norm{\deltawwt{t}}^2 \!\!=\! \F(t, \norm{\deltatGeneral{0}}^2\!\!, \seq{l}{H_l}\!, \seq{l}{G_l}), \forall t\in [T].\label{eq.result_K1} 
\end{align}
For the simple case described by \cref{eq.simple_case1,eq.simple_case4,eq.simple_case5}, we have
\begin{align}\label{eq.simple_K1}
    &\E \norm{\deltawwt{t}}^2 = H^{t} \norm{\deltatGeneral{0}}^2+ \frac{1-H^t}{1-H} G,
\end{align}
where $H\defeq  (1-\alpha)^2 + \frac{\alpha^2 (p+1)}{mn},\ G \defeq \frac{p\alpha^2 \sigma^2}{mn} + \frac{\alpha^2 (p+1)}{mn}\cdot \gammaSquareAvg$.
\end{theorem}
We relegate the proof of \cref{th.K1} to the supplemental material \cite[\cref{app.proof_K1}]{SupplementalMaterial2024}. In what follows, two important insights for \cref{th.K1} are in order from the perspectives of model initialization effects and data heterogeneity/noise.

\textbf{Insight~1)~Effect of model initialization: A good initial/pre-trained model helps, but its effect attenuates as the number of communication rounds increases and it cannot address the data heterogeneity challenges.}
In \cref{th.K1}, $\norm{\deltatGeneral{0}}^2$ denotes the model error induced by the model initialization $\wInit$ (cf. \cref{eq.def_delta}). 
\cref{th.K1} shows that starting from a good initialization (e.g., a pre-trained model) reduces the training time required to reach a target error rate. 
The reason is that a good initial/pre-trained model is usually closer to the target solution $\vwTrue$ than a random model initialization.
Thus, $\norm{\deltatGeneral{0}}$ will be small and it helps to reduce the model error. 
This result theoretically explains previously observed experimental results that using pre-trained models as the initialization for FL accelerates the training process~\citep{chen2022pre,nguyen2023where}.
Meanwhile, we note that the coefficient of $\norm{\deltatGeneral{0}}^2$ decreases as $t$ increases when the learning rate is relatively small.\footnote{In \cref{eq.simple_K1}, $H<1$ when $\step < \frac{2}{1+\frac{p+1}{mn}}$.} 
This means that the effect of the pre-trained model diminishes as the number of communication rounds increases. 
As $t \to \infty$, the first term in \cref{eq.simple_K1} asymptotically goes to 0, signifying a vanishing effect of the pre-trained model. 
This finding is consistent with existing analyses in FL, suggesting that pre-training becomes unnecessary with a sufficiently long training~\citep{gu2022and}.
In addition, \cref{th.K1} shows that the error induced by data noise and heterogeneity is not affected by the model initialization. 
This means that even a good initial/pre-trained model cannot alleviate the problems caused by heterogeneous data, which theoretically confirms prior experimental observations~\citep{chen2022pre}.



\smallskip
\textbf{Insight~2)~Effect of noise and heterogeneity: Errors arising from data noise and heterogeneity accumulate as the number of communication rounds increases, but eventually converge to an asymptotic limit.} 
In \cref{eq.simple_K1}, the coefficient of the second error term attributed to data noise and heterogeneity ($G$) is expressed as $\frac{1-H^t}{1-H}=1 + H + H^2 + \cdots + H^{t-1}$.
This implies that the error induced by data noise and heterogeneity accumulates as the value of $t$ increases. 
Meanwhile, this error term is bounded from above and it eventually converges to $\frac{1}{1-H}G$ as $t \to \infty$.
This aligns with the empirical observations that FL algorithms remain effective, despite the occurrence of model drift resulting from data heterogeneity~\citep{wang2022unreasonable,Li2020convergence,Li2020fedprox,yang2020achieving}.


\begin{figure}[t]
    \centering
    \begin{minipage}{0.45\textwidth}
        \centering
        \includegraphics[width=0.8\textwidth]{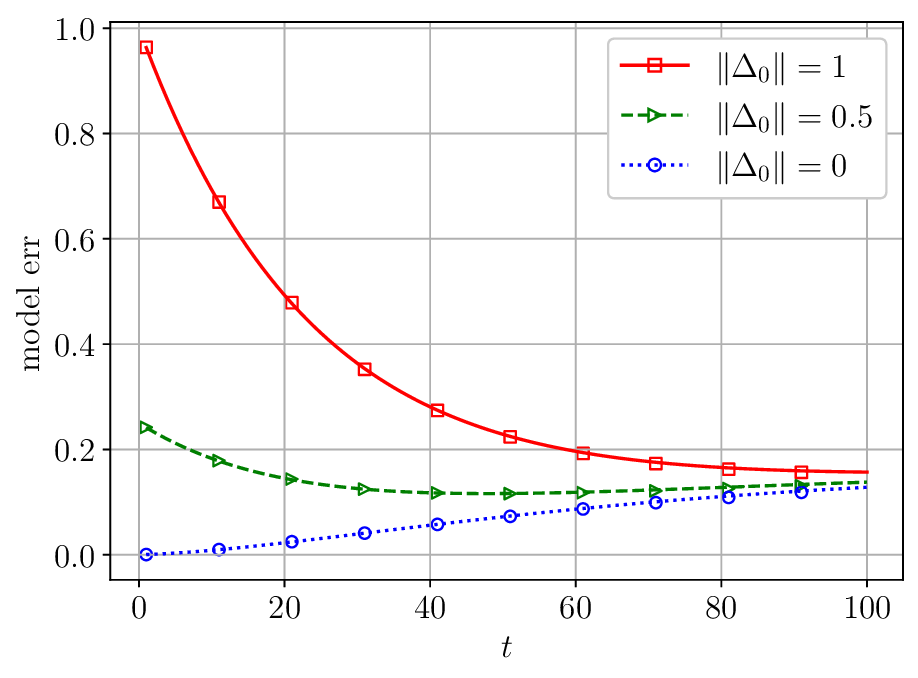}
        \caption{Experimental and analytical values of the model error w.r.t. $t$ where $K=1$, $m=3$, $p=200$, $\nij{i}{t}=50$, $s=5$, $\norm{\gammait{i}{t}} = 0.5$, and $\sigmaij{i}{t} = 0.7$ for all $i,t$. Each marker point is the experimental value by averaging over 20 simulation runs. The curves are theoretical values of \cref{th.K1}. (All markers are close to curves, which validates \cref{th.K1}.)}
        \label{fig.change_t}
    \end{minipage}\hfill
    \begin{minipage}{0.45\textwidth}
        \centering
        \includegraphics[width=0.8\textwidth]{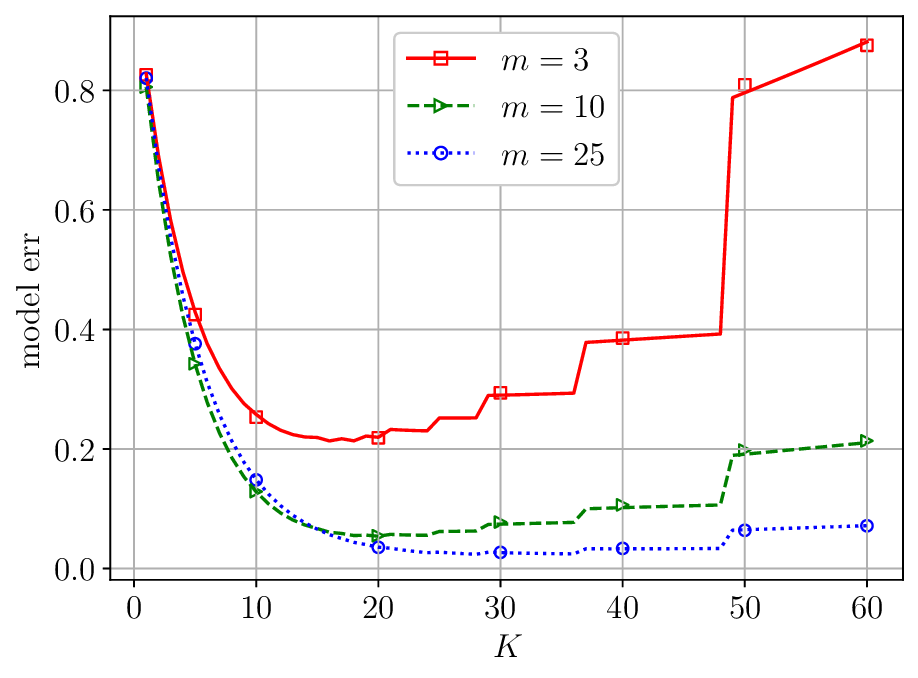}
        \caption{Experimental and theoretical values of the model error w.r.t. $K$ where $t=5$, $s=5$, $p=200$, $\norm{\deltatGeneral{0}}=1$, $\nij{i}{t}=144$, $\norm{\gammait{i}{t}}=0.5$, and $\sigmaij{i}{t} = 0.7$ for all $i,t$. Each marker point is the experimental value by averaging over 20 simulation runs. The curves are theoretical values of \cref{th.multi_K}. The lowest points of the three curves for cases $m=3,10,25$ are located at $K=15,19,27$, respectively.}
        \label{fig.change_K}
    \end{minipage}
\end{figure}

\smallskip
\textbf{Experiments.}
We perform simulations to illustrate the influence of model initialization in FL. 
The experimental setup is as follows: $K=1$, $m=3$, $p=200$, $\nij{i}{t}=50$, $s=5$, $\norm{\gammait{i}{t}} = 0.5$, and $\sigmaij{i}{t} = 0.7$ for all $i,t$.
Each marker point denotes the outcome of simulations averaged over 20 simulation trials.
In \cref{fig.change_t}, we plot the model error with respect to (w.r.t.) $t$ for three different pre-trained models: $\norm{\deltatGeneral{0}}=1$ (red solid line with markers ``$\square$''), $\norm{\deltatGeneral{0}}=0.5$ (green dashed line with markers ``$\triangleright$''), and $\norm{\deltatGeneral{0}}=0$ (blue dotted line with markers ``$\circ$''). 
Generally, the simulation demonstrates the tightness of our theoretical findings and confirms our two insights mentioned above. 
The blue curve, indicative of the smallest initial model error, initially outperforms the other two curves. However, this performance gap diminishes as time progresses. This observed trend aligns with our insights into the impact of model initialization. 
Conversely, as the blue curve originates from the ideal solution, its upward trend with respect to $t$ is solely attributed to noise and heterogeneity. This observation further validates our understanding of the influence of noise and heterogeneity.




\subsection{The General \texorpdfstring{$K<\infty$}{K} Case}

Similar to \cref{eq.def_Ht,eq.def_Gt}, we define $\mathcal{J}_l, \mathcal{Q}_l\in \mathds{R}$. 
The expressions of $\mathcal{J}_l$ and $\mathcal{Q}_l$ only contain $\nij{i}{t}$, $p$, $\step$, $\gammait{i}{t}$, $\deltatGeneral{0}$, and the number of local steps $K$. 
The formal definitions are provided in \cref{eq.def_Jt,eq.def_Qt} at the beginning of supplemental material \cite[\cref{app.proof_multi_K}]{SupplementalMaterial2024} due to space limitation.

\begin{theorem}\label[theorem]{th.multi_K}
When $K <\infty$, we have
\begin{align}
    \E \norm{\deltat{t}}^2 = \F\left(t, \norm{\deltatGeneral{0}}^2, \seq{l}{\mathcal{J}_l},\seq{l}{\mathcal{Q}_l}\right).\label{eq.multi_K}
\end{align}
For the simple case described by \cref{eq.simple_case1,eq.simple_case4,eq.simple_case5} and by further letting $\gammaSquareAvg=0$, we have
\begin{align}\label{eq.simple_Kc}
    \E \norm{\deltat{t}}^2 = \mathcal{J}^t \norm{\deltatGeneral{0}}^2 + \frac{1- \mathcal{J}^t}{1-\mathcal{J}} \cdot \frac{\alpha^2 p \sigma^2}{m\tilde{n}} \cdot \frac{1-\mathcal{A}^K}{1-\mathcal{A}},
\end{align}
where $\tilde{n}\defeq \left\lfloor n/K\right\rfloor,\  \mathcal{A}\defeq (1-\alpha)^2 + \frac{\alpha^2 (p+1)}{\tilde{n}},\ \mathcal{J}\defeq \frac{\mathcal{A}^K+(m-1)(1-\alpha)^{2K}}{m}$.
\end{theorem}

The proof for \cref{th.multi_K} is provided in the supplemental material \cite[\cref{app.proof_multi_K}]{SupplementalMaterial2024}. 
Building upon the insights gained from \cref{th.multi_K}, we have the following discussions concerning the impact of the local update step $K$.

\smallskip
\textbf{Insight~3)~Effect of the local update step number $K$: The optimal choice of finite $K$ sometimes exists.} 
In \cref{eq.simple_Kc}, the local update step number $K$ together with several other factors simultaneously influence two error terms.
Therefore, the optimal choice of $K$ is dependent on other configurations, such as the number of communication round $t$, $\norm{\deltatGeneral{0}}^2$ (determined by the model initialization), and the noise denoted by $\sigma^2$. 
Through an analysis of how \cref{eq.simple_Kc} evolves with $K$, we establish the following proposition for the optimal choice of $K$:

\begin{proposition}\label{prop.opt_K}
The existence of an optimal choice of $K$ (defined by $K_{\text{opt}}$) for \cref{eq.simple_Kc} in different cases are as follows:

(1) A finite $K_{\text{opt}}$-value must exist  when $\tilde{n}$ is fixed (i.e., $n$ is determined by $K\tilde{n}$), $\alpha$ is sufficiently small\footnote{When $\alpha < \frac{2}{1 + \frac{p}{\tilde{n}}}$, we have $\mathcal{A}< 1$, and thus $\mathcal{J}<\frac{1 + (m-1)}{m}=1$.}, and $t\to \infty$.

(2) A finite $K_{\text{opt}}$-value does not exist (i.e., $K_{\text{opt}}=\infty$) when $\tilde{n}$ is fixed, $\alpha$ is sufficiently small, and $\sigma = 0$.

(3) When $n$ is fixed (i.e., $\tilde{n}$ is determined by $\lfloor n / K \rfloor$), $t<\infty$, $\alpha\leq 0.1$, $m\geq 3$, and $\sigma = 0$, if we neglect the difference between $\lfloor n / K \rfloor$ and $n/K$, then
\begin{align}
    \frac{n}{p+1}\left(\frac{2}{\alpha}-1\right) \leq {K}_{\text{opt}} \leq \frac{n}{p+1}\frac{(m-2)}{\alpha^3}.\label{eq.temp_092701}
\end{align}
\end{proposition}
In \cref{prop.opt_K}, we show that the optimal and finite $K$-value only exists in some cases, whose value depends on other parameters in one specific problem instance. 
For example, the upper bound of ${K}_{\text{opt}}$ in Eq.~\eqref{eq.temp_092701} indicates that \textbf{the optimal $K$ may increase when the number of agents $m$ increases}.
This discovery offers a theoretical explanation of the experimental controversy, wherein switching to local update steps yields divergent outcomes for various tasks; some exhibit improved performance, while others do not~\citep{lin2019don,ortiz2021trade,gu2022and}.
Proof of \cref{prop.opt_K} is provided in \cref{app.proof_opt_K}.

\smallskip
\textbf{Experiments.}
Following a similar setting in \cref{fig.change_t}, we plot the model error against the local steps $K$ when $\nij{i}{t}$ is fixed in \cref{fig.change_K}. 
These three curves in \cref{fig.change_K} correspond to different values of $m$. We can see that each of the three curves in \cref{fig.change_K} has a minimum. The lowest points of the three curves for cases $m=3,10,25$ are located at $K=15,19,27$ (i.e., $K_{\text{opt}}$), respectively. 
This phenomenon supports our insights that the optimal $K$ only exists ``sometimes'' and may increase w.r.t. $m$.



\subsection{The \texorpdfstring{$K=\infty$}{k=infty}  Case (Convergence in Local  Update)}\label{subsubsec.converge_K}

We define the following short-hand notations:
\begin{align*}
    &\bm{g}_{l}^{K=\infty} \defeq  \F\left(l, \deltatGeneral{0}, \seq{t}{A_t},\seq{t}{\bm{b}_t}\right),\numberthis \label{eq.def_F_inf}\\
    &A_t \defeq \frac{1}{\sum_{i'\in [m]}\nij{i'}{t}}\sum_{i'\in[m]}\nij{i'}{t} \left(1-\frac{\nij{i'}{t}}{p}\right) ,\numberthis \label{eq.def_At}\\
    &\bm{b}_t \defeq \frac{1}{\sum_{i'\in [m]}\nij{i'}{t}}\sum_{i'\in[m]}\nij{i'}{t}\cdot \frac{\nij{i'}{t}}{p}\gammait{i'}{t} .\numberthis \label{eq.def_bt}\\
    &C_t \defeq \mySumSquareP{\sum_{i=1}^m \left(\nij{i}{t}^2 \left(1-\frac{\nij{i}{t}}{p}\right)  \right)} \\
    &\qquad + \mySumSquareP{\sum_{i\neq j} \nij{i}{t}\nij{j}{t}\left(1-\frac{\nij{i}{t}}{p}\right)\left(1-\frac{\nij{j}{t}}{p}\right)},\numberthis \label{eq.def_Ct}
\end{align*}
\begin{align*}
    &D_t \defeq \mySumSquareP{\sum_{i\in [m]}\frac{\nij{i}{t}^3 \sigmaij{i}{t}^2}{p-\nij{i}{t}-1}+\frac{\nij{i}{t}^3}{p}\norm{\gammait{i}{t}}^2}\\
    &+ \mySumSquareP{\sum_{i\in [m]}\sum_{j\in [m]\setminus\{i\}} \left(\frac{\nij{i}{t}^2\nij{j}{t}^2}{p^2} \gammait{i}{t}\T\gammait{j}{t} \right) } + \\
    &\mySumSquareP{\sum_{i\in [m]}\sum_{j\in [m]\setminus\{i\}} 2 \frac{\nij{j}{t}^2}{p}\nij{i}{t}\left(1 - \frac{\nij{i}{t}}{p}\right)\gammait{j}{t}\T \bm{g}_{t-1}^{K=\infty}}.\numberthis \label{eq.def_Dt}
\end{align*}


\begin{theorem}\label[theorem]{th.main}
In the over-parameterized (OP) regime, i.e., $p>\max \nij{i}{t}+1$, it holds that
\begin{align}
    \E \norm{\deltawt{t}}^2 \!\!\!=\!\! \F(t, \norm{\deltatGeneral{0}}^2\!\!\!, \seq{l}{C_l}, \seq{l}{D_l}),\forall t\in [T].\label{eq.inf_K_main}
\end{align}
In the under-parameterized (UP) regime, i.e., $p<\min \nij{i}{t}-1$, it holds that
\begin{align}\label{eq.underparameterized}
    \E \norm{\deltawt{t}}^2 \!\!=\! \norm{\mySumP{\nij{i}{t}\gammait{i}{t}}}^2 \!\!\!+\! \mySumSquarePP{\frac{\nij{i}{t}^2 p \sigmaij{i}{t}^2}{\nij{i}{t}-p-1}}.
\end{align}
For the simple case described by \cref{eq.simple_case1,eq.simple_case4,eq.simple_case5}, it holds that
\begin{align}\label{eq.main_thm_simple}
    \E \norm{\deltawt{t}}^2 = \begin{cases}
    C^{t} \norm{\deltatGeneral{0}}^2+ \frac{1-C^t}{1-C} D & \text{ if OP},\\
    \frac{p\sigma^2}{m(n-p-1)}&\text{ if UP},
    \end{cases}
\end{align}
where
\begin{align}
    C &\defeq  \frac{1}{m}\left(1-\frac{n}{p}\right)+\frac{m-1}{m}\left(1-\frac{n}{p}\right)^2 < 1,\\
    D &\defeq \frac{n\sigma^2}{m(p-n-1)} + \frac{n}{p}\gammaSquareAvg.\label{eq.temp_092704}
\end{align}
\end{theorem}
We provide a proof sketch of \cref{th.main} in \cref{sec.proof_sketch}. 
The complete proof is in the supplemental material \cite[\cref{app.proof_main}]{SupplementalMaterial2024}.

\smallskip
\textbf{Insight~4)~Benign overfitting exists in FL, and the ``null risk'' can be alleviated by using more communications rounds.} In the over-parameterized case of \cref{eq.main_thm_simple}, the term $D$ decreases when $p$ increases. Thus, when the term $D$ dominates (e.g., when noise and/or heterogeneity is large, or $t$ is large), the generalization performance of FL in this case will benefit from more parameters when overfitted. This validates the ``double-descent'' or benign overfitting phenomenon in the literature of the classical (single-task single-agent) linear regression (e.g. \citet{belkin2020two}). For the comparable Gaussian models we used, the expectation of the model error of such a classical (single-task single-agent) linear regression is
\begin{align}\label{eq.overfit_classical}
    (1-\frac{n}{p})\norm{\deltatGeneral{0}}^2+\frac{n\sigma^2}{p-n-1}.
\end{align}
By \cref{eq.overfit_classical} and related literature (e.g., \citet{ju2020overfitting}), the classical linear regression suffers from the ``null risk'' (i.e., converges to the initial error) when $p\to \infty$. 
However, for the FL result in \cref{eq.main_thm_simple}, we can see that the ``null risk'' term $\norm{\deltatGeneral{0}}^2$ can be alleviated by the coefficient $C^t$, which approaches zero when $t\to\infty$. 
In other words, for fixed $n$, when $p\to \infty$, as long as we let $t\to \infty$ in a faster speed (e.g., $t=p\log p$, proved in \cref{le.limit_t} in supplemental material \cite[\cref{app.useful_lemmas}]{SupplementalMaterial2024}), then the null risk term $C^t \norm{\deltatGeneral{0}}^2\to 0$,  which implies that using more communication rounds in FL (i.e., larger $t$) mitigates the null risk, thus ``enhancing'' the benefits of overfitting.

\begin{figure}
    \centering
    \includegraphics[width=0.6\textwidth]{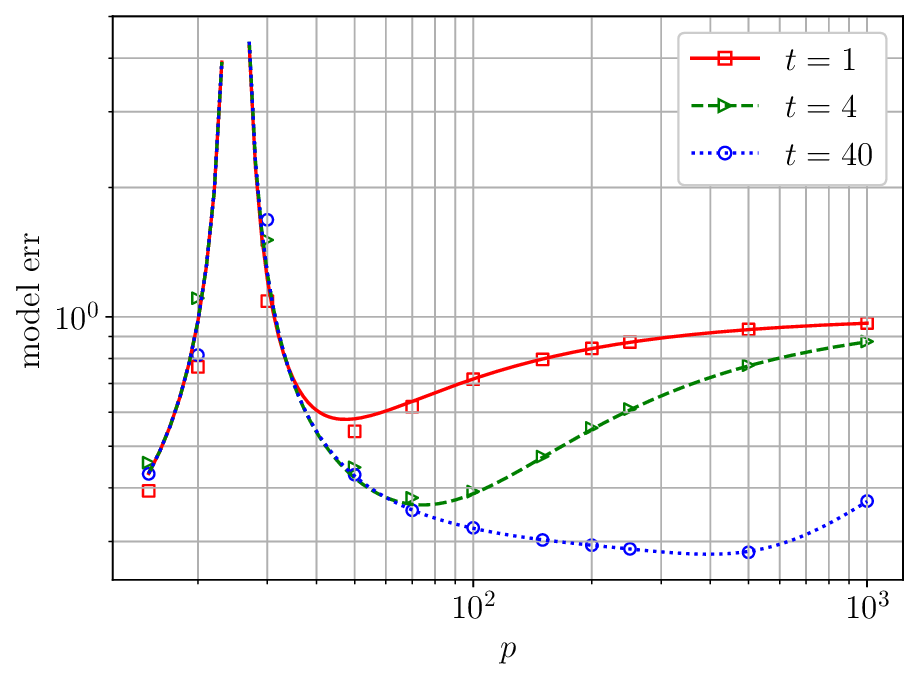}
    \caption{Experimental and theoretical values of the model error w.r.t. $p$ where $m=3$, $s=5$, $\nij{i}{t}=25$, $\norm{\deltatGeneral{0}}=1$, $\norm{\gammait{i}{t}}=0.5$, and $\sigmaij{i}{t} = 0.7$ for all $i,t$. Each marker is the experimental value by averaging over 20 simulation runs. The curves are drawn purely by the theoretical values from \cref{th.main}.}
    \label{fig.change_p}
\end{figure}

\smallskip
\textbf{Experiments.}
In \cref{fig.change_p}, we present a plot of model error against $p$ in both the underparameterized regime ($p< n=25$) and overparameterized regime ($p> n = 25$) for three cases with $t=1$, $t=4$, and $t=40$. 
The curves represent theoretical values derived from \cref{th.main}, while each marker signifies the average of 20 simulation trials.
It is observed that all three curves exhibit a decreasing trend at the initial phase of the overparameterized regime, confirming the presence of benign overfitting. 
Additionally, when comparing these three cases, the curve for $t=40$ (indicated by a blue dotted line with markers " $\circ$ ") has a more substantial and broader descent. 
This observation confirms our insight that a larger $t$-value enhances the benefits of overfitting in FL.%

\section{Proof Sketch of Proposition~\ref{prop.opt_K}}\label{app.proof_opt_K}
(1) Since $\tilde{n}$ is fixed, then $\mathcal{A}$ does not change with $K$. When $t\to \infty$, the value of \cref{eq.simple_Kc} becomes
\begin{align}\label{eq.temp_092601}
    \frac{1}{1-\mathcal{J}} \frac{\alpha^2 p \sigma^2}{mn}\cdot \frac{1-\mathcal{A}^K}{1-\mathcal{A}}.
\end{align}
The only component related to $K$ in \cref{eq.temp_092601} is $\frac{1-\mathcal{A}^K}{1-J}$, thus $K_{\text{opt}}=\argmin_K \frac{1-\mathcal{A}^K}{1-J}$. Notice that for any finite $K$, we must have
\begin{align*}
    \mathcal{A}^K = \left(1-\alpha)^2 + \frac{\alpha^2 (p+1)}{\tilde{n}}\right)^K > (1-\alpha)^{2K}.
\end{align*}
Thus, we have
\begin{align*}
    \mathcal{J}= \frac{1}{m} \mathcal{A}^K + \frac{m-1}{m}(1-\alpha)^{2K}< \mathcal{A}^K,
\end{align*}
which implies that $\frac{1-\mathcal{A}^K}{1-\mathcal{J}} < 1$ for any finite $K$. Meanwhile, $\lim\limits_{K\to \infty}\frac{1-\mathcal{A}^K}{1-\mathcal{J}}= 1$. Thus, $K_{\text{opt}}$ should be finite.

(2) Since $\tilde{n}$ is fixed, then $\mathcal{A}$ does not change with $K$. When $\sigma=0$, \cref{eq.simple_Kc} becomes $\mathcal{J}^t \norm{\deltatGeneral{0}}^2$. Notice that $\mathcal{J}$ is strictly monotone decreasing w.r.t. $K$. Therefore, $K_{\text{opt}}=\infty$.

(3) Since we use $n / K$ to replace $\lfloor n / K \rfloor$, we have ${K}_{\text{opt}}=\argmin_{K}f(K)$ where
\begin{align*}
    f(K)\defeq \left((1-\alpha)^2 + K\frac{\alpha^2 (p+1)}{n}\right)^K + (m-1)(1-\alpha)^{2K}.
\end{align*}
Calculating the derivative, we have $\frac{\partial f(K)}{\partial K} =$
\begin{align*}
    & \frac{\alpha^2 (p+1)}{n} \left((1-\alpha)^2 + K\frac{\alpha^2 (p+1)}{n}\right)^K \ln \left((1-\alpha)^2 + K\frac{\alpha^2 (p+1)}{n}\right) \\
    &+ (m-1)(1-\alpha)^{2K}\ln ((1-\alpha)^2).\numberthis \label{eq.temp_092602}
\end{align*}
When $\left((1-\alpha)^2 + K\frac{\alpha^2 (p+1)}{n}\right) < 1$, we have $\frac{\partial f(K)}{\partial K} < 0$.

For any $\delta>0$, when
\begin{align*}
    &\left((1-\alpha)^2 + K\frac{\alpha^2 (p+1)}{n}\right) > 1+\delta,\\
    &\frac{\alpha^2 (p+1)}{n} (1+K\delta) \ln (1+\delta) > (m-1)\ln \frac{1}{(1-\alpha)^2},
\end{align*}
we have $\text{\cref{eq.temp_092602}}>0$. (Notice that we utilize the face that $(1-\alpha)^{2K}<1$ and $(1+\delta)^K \geq 1+K\delta$.) Solving those inequalities by further letting $\ln (1+\delta) = \ln \frac{1}{(1-\alpha)^2}$, we thus have
\begin{align*}
    &\frac{n}{(p+1)}\left(\frac{2}{\alpha}-1\right)\leq {K}_{\text{opt}} \\
    &\leq \frac{n}{\alpha^2(p+1)}\cdot \max\left\{(2\alpha-\alpha^2)\left(1+\frac{1}{(1-\alpha)^2}\right),\ (m-2)\frac{(1-\alpha)^2}{2\alpha - \alpha^2}\right\}.
\end{align*}
When $\alpha \leq 0.1$ and $m\geq 3$, we can further relax the above inequality as
\begin{align*}
    \frac{n}{p+1}\left(\frac{2}{\alpha}-1\right) \leq {K}_{\text{opt}} \leq \frac{n}{p+1}\frac{(m-2)}{\alpha^3}.
\end{align*}

 


\section{Proof Sketch of Theorem~\ref{th.main}}\label{sec.proof_sketch}
We provide a sketched proof of \cref{eq.inf_K_main} here.
For any $i\in [m]$, we define $\Pij{i}{t}\in \mathds{R}^{p\times p}$ as
\begin{align}\label{eq.def_Pij_another}
    \Pij{i}{t}\defeq \XX{i}{t}\left(\XX{i}{t}\T\XX{i}{t}\right)^{-1} \XX{i}{t}\T.
\end{align}
(We know $\Pij{i}{t}$ is an orthogonal projection since $\Pij{i}{t}\Pij{i}{t}=\Pij{i}{t}$ and $\Pij{i}{t}\T = \Pij{i}{t}$.)
In the overparameterized situation, after each agent trains to converge, we have
\begin{align}
    \wij{i}{t} =& \Pij{i}{t}\wijTrue{i}{t}+ (\iMatrix{p} - \Pij{i}{t}) \wavg{t-1} \nonumber\\
    & + \XX{i}{t}\left(\XX{i}{t}\T\XX{i}{t}\right)^{-1} \ee{i}{t}.\label{eq.temp_082401_another}
\end{align}



We thus have
\begin{align*}
    \deltawt{t} =&\vwTrue-\wavg{t}\ \text{ (by \cref{eq.def_delta})}\\
    =&\mySum\nij{i}{t}\left(\Pij{i}{t} \gammait{i}{t} + (\iMatrix{p}\right.\\
    &\left.-\Pij{i}{t})\deltawt{t-1}- \XX{i}{t}\left(\XX{i}{t}\T\XX{i}{t}\right)^{-1} \ee{i}{t}\right).\numberthis \label{eq.delta_t_expression_another}
\end{align*}
Thus, we can write $\E_t \norm{\deltawt{t}}^2$ into the inner product between the terms in \cref{eq.delta_t_expression_another}.
The key part of the proof is to calculate those inner product terms. 
The terms that involve only one agent can be calculated using the known results in the literature, e.g.,
\begin{align*}
    &\E_t\norm{\Pij{i}{t}\gammait{i}{t}}^2 = \frac{\nij{i}{t}}{p}\norm{\gammait{i}{t}}^2.\numberthis \label{eq.temp_041401}
\end{align*}
The remaining terms in $\E_t \norm{\deltawt{t}}^2$ involve different agents $i\neq j$, which are unique to FL and not seen in the literature. The key step is to prove
\begin{align}
    \E_{\Pij{i}{t}} \left[\Pij{i}{t}\deltawt{t-1}\right] = \frac{\nij{i}{t}}{p} \deltawt{t-1}.\label{eq.key_step}
\end{align}
We provide an intuition of \cref{eq.key_step} along with a geometric interpretation in \cref{fig.geometric} at the end of this section. By using \cref{eq.key_step}, the terms involving different agents $i\neq j$ can be calculated, e.g.,
\begin{align*}
    \E_t \left[\gammait{j}{t}\T \Pij{j}{t}\Pij{i}{t}\gammait{i}{t}\right] = \frac{\nij{i}{t}\nij{j}{t}}{p^2} \gammait{j}{t}\T \gammait{i}{t}.
\end{align*}

With the above equations, we thus have
\begin{align}
    \E \norm{\deltawt{t}}^2 = C_t \cdot  \E\norm{\deltawt{t-1}}^2 + D_t,\label{eq.temp_083101_another}
\end{align}
where $C_t$ denotes the coefficient of $\norm{\deltawt{t-1}}^2$ and $D_t$ denotes the remaining parts. The specific expressions of $C_t$ and $D_t$ are in \cref{eq.def_Ct,eq.def_Dt}.
Applying \cref{eq.temp_083101_another} recursively, we thus have \cref{eq.inf_K_main}.

\begin{figure}
    \centering
    \includegraphics[width=0.6\textwidth]{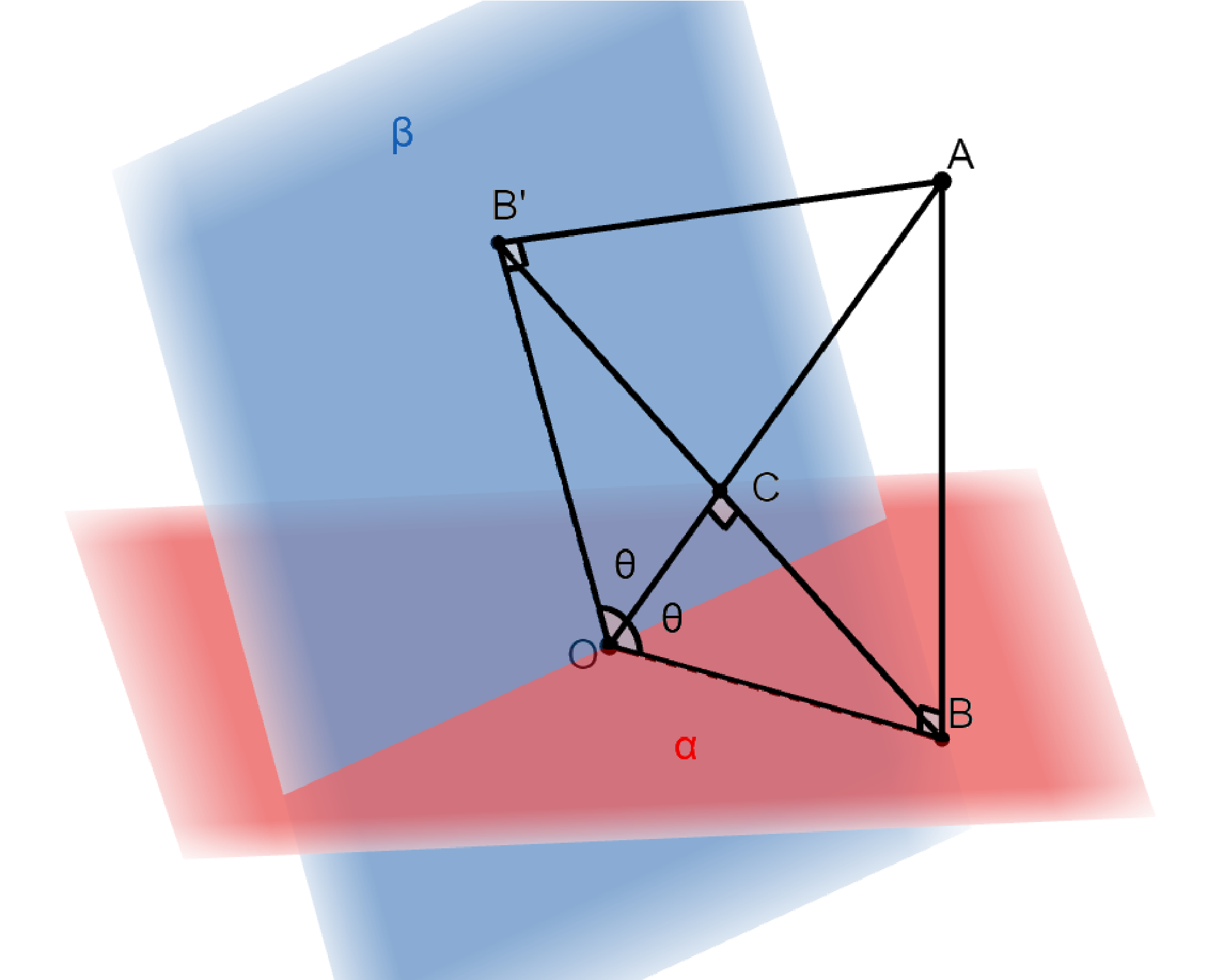}
    \caption{Geometric interpretation of \cref{eq.key_step}.}
    \label{fig.geometric}
\end{figure}

\emph{Intuition of \cref{eq.key_step}:} We use \cref{fig.geometric} to help illustrating the intuition. In \cref{fig.geometric}, the vector $\overrightarrow{OA}$ denotes $\deltawt{t-1}$, the plane $\alpha$ denotes the space spanned by the columns of $\XX{i}{t}$. Notice that $\Pij{i}{t}\deltawt{t-1}$ represents result of projecting $\deltawt{t-1}$ to the column space of $\XX{i}{t}$, i.e., the vector $\overrightarrow{OB}$ in \cref{fig.geometric}. Therefore, in \cref{fig.geometric}, calculating $\E\limits_{\Pij{i}{t}}\Pij{i}{t}\deltawt{t-1}$ means calculating the average of $\overrightarrow{OB}$ when the hyper-plane $\alpha$ rotating around the point $O$. Notice that $\overrightarrow{OB}=\overrightarrow{OC}+\overrightarrow{CB}$ where $\overrightarrow{OC}$ and $\overrightarrow{CB}$ are the parallel and perpendicular components of $\overrightarrow{OB}$ w.r.t. $\overrightarrow{OA}$, respectively. Because of the rotational symmetry of the hyper-plane $\alpha$ (due to the rotational symmetry of each column of $\XX{i}{t}$), we know that all the perpendicular components are cancelled out while only the parallel components remain in the averaging process. In other words, for any hyper-plane $\alpha$, there exists a symmetrical (w.r.t. $\overrightarrow{OA}$) hyper-plane $\beta$ with the same probability density such that the projection of $\overrightarrow{OA}$ to the hyper-plane $\beta$, named $\overrightarrow{OB'}$, has the same parallel component $\overrightarrow{OC}$ but the opposite perpendicular component $\overrightarrow{CB'}=-\overrightarrow{CB}$. Thus, we only need to calculate the average of the parallel component $\overrightarrow{OC}$, whose length equals $\cos \theta \abs{\overrightarrow{OB}}$, where $\theta=\angle AOB$ is defined as the angle between $\deltawt{t-1}$ and $\Pij{i}{t}\deltawt{t-1}$ (i.e., the angle between $\deltawt{t-1}$ and the hyperplane spanned by the columns of $\XX{i}{t}$ as
\begin{align}
    \theta \defeq \arccos \frac{\Pij{i}{t}\deltawt{t-1}}{\norm{\deltawt{t-1}}}.\label{eq.temp_072102}
\end{align}
Also notice that $\abs{\overrightarrow{OB}}=\cos \theta \abs{\overrightarrow{OA}}$. Thus, the length of the parallel component equals $\abs{\overrightarrow{OC}}=\cos^2\theta \abs{\overrightarrow{OA}}$. Therefore, we have $\E \overrightarrow{OC} = \E \cos^2 \theta \overrightarrow{OA}=\frac{\nij{i}{t}}{p} \deltawt{t-1}$. The last equation uses a known result in literature just as \cref{eq.temp_041401}.

\section{Conclusion} \label{sec:conclusion}


In this paper, we have precisely quantified the influence of data heterogeneity and the local update process on the generalization performance of FedAvg-type algorithms. 
Specifically, we undertook a thorough theoretical examination of FL's generalization performance utilizing a linear model, which yields closed-form expressions for the model error.
Our analysis rigorously assesses the impact of local update steps (represented by $K$) across three distinct settings ($K=1$, $K<\infty$, and $K=\infty$), elucidating how generalization performance evolves with the progression of rounds, denoted as $t$. 
Additionally, our investigation yields a comprehensive understanding of how various configurations, including the number of model parameters $p$, the number of training samples $n$, the local steps $K$, and the total communication round $t$, contribute to the overall FL generalization performance. 
This, in turn, unveils new insights, such as the phenomenon of benign overfitting, optimal local steps, and the impact of a good model initialization, with practical implications for the implementation of FL.

\section{Acknowledgement}

This work has been supported in part by NSF grants NSF AI Institute (AI-EDGE) CNS-2112471, CNS-2312836, CNS-2106933, CNS-2106932, CNS-2312836, CNS-1955535, CNS-1901057, ECCS-2113860 and 2324052, CAREER CNS-2110259, and ECCS-2331104, by Army Research Office under Grant W911NF-21-1-0244 and was sponsored by the Army Research Laboratory and was accomplished under Cooperative Agreement Number W911NF-23-2-0225, by DARPA Grant D24AP00265. The views and conclusions contained in this document are those of the authors and should not be interpreted as representing the official policies, either expressed or implied, of the Army Research Laboratory or DARPA or the U.S. Government. The U.S. Government is authorized to reproduce and distribute reprints for Government purposes notwithstanding any copyright notation herein.

\bibliographystyle{unsrt}

\bibliography{references,FL}
\newpage
\appendix
\begin{center}
\textbf{\large Supplemental Material}
\end{center}

We give a table to summarize the content of the supplemental material.
\begin{table}[h!]
\centering
\begin{tabular}{||c | c | c||} 
 \hline
 {\bf Section} & {\bf Content}\\ 
 \hline\hline
 \cref{app.useful_lemmas} & some useful lemmas as technical tools\\
 \hline
 \cref{app.proof_K1} & proof of \cref{th.K1} for $K=1$\\
 \hline
 \cref{app.proof_multi_K} & proof of \cref{th.multi_K} for $K<\infty$\\
 \hline
 \cref{app.proof_main} & proof of \cref{th.main} for $K=\infty$\\
 \hline
 \cref{app.notation_table} & a table for some important notations\\
 \hline
\end{tabular}
\caption{Outline of the supplemental material.}
\label{table.outline}
\end{table}

\section{Useful Lemmas}\label{app.useful_lemmas}

In this section, we provide some useful lemmas. Specifically, \cref{le.limit_t} is used to support the claim of the convergence speed in Insight 4. \cref{le.IW,le.bias,le.XXXX} are some results about the Gaussian random matrices that can be found in the literature. We want to highlight \cref{le.cross_term_exact_value} as part of our technical novelty, which gives the exact values of terms related to the projection formed by each agent's training inputs. \cref{le.model_error} is used to justify the definition of model error.

\begin{lemma}\label[lemma]{le.limit_t}
Recalling the definition of $C$ in \cref{eq.temp_092704}, we have
\begin{align*}
    \lim_{t=p\ln p,\ p\to \infty} C^t = 0.
\end{align*}
\end{lemma}
\begin{proof}
We have $C^t\geq 0$ and
\begin{align*}
    C^t \leq& \left(1-\frac{n}{p}\right)^t\ \text{ (since $C\leq \left(1-\frac{n}{p}\right)$ because $\left(1-\frac{n}{p}\right)^2 \leq \left(1-\frac{n}{p}\right)$)}\\
    =&\left(1+\frac{1}{\frac{p}{n}-1}\right)^{-t}\ \text{ (since $1-\frac{n}{p}=\frac{1}{1+\frac{1}{\frac{p}{n}-1}})$}\\
    =& \left(1+\frac{1}{\frac{p}{n}-1}\right)^{-p\ln p}\ \text{ (since $t = p\ln p$)}\\
    =& \left(1+\frac{1}{\frac{p}{n}-1}\right)^{-\frac{p}{n}\cdot n \cdot \ln p}\\
    \leq & \left(1+\frac{1}{\frac{p}{n}-1}\right)^{-\left(\frac{p}{n}-1\right)\cdot n \cdot \ln p}.
\end{align*}
Notice that
\begin{align*}
    \lim_{p\to \infty} \left(1+\frac{1}{\frac{p}{n}-1}\right)^{-\left(\frac{p}{n}-1\right)\cdot n \cdot \ln p} = \lim_{p\to \infty} e^{-n\ln p} = 0,
\end{align*}
where we use the fact that $\lim_{x\to \infty}(1 + x^{-1})^x=e$. The result of this lemma thus follows by the squeeze theorem.
\end{proof}

The result of the following lemma can be found in the literature (e.g., \cite{belkin2020two,ju2022theoretical}).
\begin{lemma}\label[lemma]{le.bias}
Consider a random matrix $\mathbf{K}\in \mathds{R}^{p\times n}$ where $p$ and $n$ are two positive integers and $p>n+1$. Each element of $\mathbf{K}$ is \emph{i.i.d.} according to standard Gaussian distribution. For any fixed vector $\bm{a}\in \mathds{R}^p$, we must have
\begin{align*}
    &\E\norm{\left(\iMatrix{p}-\mathbf{K}\left(\mathbf{K}\T\mathbf{K}\right)^{-1}\mathbf{K}\T\right)\bm{a}}^2 = \left(1-\frac{n}{p}\right)\norm{\bm{a}}^2,\\
    &\E\norm{\mathbf{K}\left(\mathbf{K}\T\mathbf{K}\right)^{-1}\mathbf{K}\T\bm{a}}^2 = \frac{n}{p}\norm{\bm{a}}^2.
\end{align*}
\end{lemma}


The following lemma can be found in Lemma~8 of \cite{ju2023generalization}.
\begin{lemma}\label[lemma]{le.IW}
Consider a random matrix $\mathbf{K}\in \mathds{R}^{a\times b}$ where $a>b+1$. Each element of $\mathbf{K}$ is \emph{i.i.d.} following standard Gaussian distribution $\mathcal{N}(0,1)$. Consider three Gaussian random vectors $\bm{\alpha},\bm{\gamma}\in\mathbf{R}^a$ and $\bm{\beta}\in\mathbf{R}^b$ such that $\bm{\alpha}\sim \mathcal{N}(\bm{0}, \sigma_{\alpha}^2\iMatrix{a})$, $\bm{\gamma}\sim \mathcal{N}(\bm{0}, \diag(d_1^2,d_2^2,\cdots,d_a^2))$, and $\bm{\beta}\sim \mathcal{N}(\bm{0}, \sigma_{\beta}^2\iMatrix{b})$. Here $\mathbf{K}$, $\bm{\alpha}$, $\bm{\gamma}$, and $\bm{\beta}$ are independent of each other. We then must have
\begin{align}
    &\E \left[(\mathbf{K}\T\mathbf{K})^{-1}\right]=\frac{\iMatrix{b}}{a-b-1},\label{eq.temp_010301}\\
    &\E \norm{\mathbf{K}(\mathbf{K}\T\mathbf{K})^{-1}\bm{\beta}}^2=\frac{b\sigma_{\beta}^2}{a-b-1},\label{eq.temp_010302}\\
    &\E \norm{(\mathbf{K}\T\mathbf{K})^{-1}\mathbf{K}\T\bm{\alpha}}^2=\frac{b\sigma_{\alpha}^2}{a-b-1},\label{eq.temp_010303}\\
    & \E \norm{(\mathbf{K}\T\mathbf{K})^{-1}\mathbf{K}\T\bm{\gamma}}^2=\frac{b\sum_{i=1}^a d_i^2}{a(a-b-1)}.\label{eq.temp_010320}
\end{align}
\end{lemma}

The following lemma can be found in \cite{bernacchia2021meta} and  Lemma~13 of \cite{ju2022theoretical}.
\begin{lemma}\label[lemma]{le.XXXX}
    Consider a random matrix $\mathbf{K}\in \mathds{R}^{a\times b}$ whose each element follows \emph{i.i.d.} standard Gaussian distribution (i.e., \emph{i.i.d.} $\mathcal{N}(0, 1)$). We mush have
    \begin{align*}
        &\E[\mathbf{K}\T \mathbf{K}]=a\iMatrix{b},\\
        &\E[\mathbf{K} \mathbf{K}\T]=b\iMatrix{a},\\
        &\E[\mathbf{K} \mathbf{K}\T \mathbf{K}\mathbf{K}\T ]=b(b+a+1)\iMatrix{a}.
    \end{align*}
\end{lemma}

\begin{lemma}\label[lemma]{le.cross_term_exact_value}
For any $i\in [m]$ and $t$, we must have
\begin{align}
    \E_{\Pij{i}{t}} \left[\Pij{i}{t}\deltawt{t-1}\right] = \frac{\nij{i}{t}}{p} \deltawt{t-1}.\label{eq.temp_072101}
\end{align}
Consequently, when $i\neq j$, we have
\begin{align*}
    \E_{\Pij{i}{t},\Pij{j}{t}}\left[\deltawt{t-1}\T \Pij{i}{t}\Pij{j}{t}\deltawt{t-1}\right] = \frac{\nij{j}{t}\nij{i}{t}}{p^2}\norm{\deltawt{t-1}}^2.
\end{align*}
\end{lemma}

\begin{proof}
Let $C\defeq \norm{\deltawt{t-1}}$. Since we are calculating expected projection of $\deltawt{t-1}$ onto the column space of $\XX{i}{t}$, by the symmetry of $\XX{i}{t}$, without loss of generality we let
\begin{align}
    \deltawt{t-1} = C \cdot \myMa{1\\0\\0\\\vdots\\0}.\label{eq.temp_081502}
\end{align}
\newcommand{\TildeX}{\tilde{\mathbf{X}}_{(i), t}}
Define
\begin{align}\label{eq.temp_081601}
    \TildeX \defeq \myMa{-1 & & &\\& 1 & &\\& & \ddots &\\& & & 1}\XX{i}{t}.
\end{align}
Since each element of $\XX{i}{t}$ follows \emph{i.i.d.} standard Gaussian distribution, we know that $\TildeX$ and $\XX{i}{t}$ has identical distributioin. Thus, we have
\begin{align}
    \int \XX{i}{t}(\XX{i}{t}\T\XX{i}{t})\XX{i}{t}\T \deltawt{t-1} d \mu(\XX{i}{t}) = \int \TildeX (\TildeX\T \TildeX)\TildeX \deltawt{t-1} d \mu(\XX{i}{t}),\label{eq.temp_081501}
\end{align}
where $\mu(\XX{i}{t})$ denotes the joint probability distribution of $\XX{i}{t}$.

By \cref{eq.temp_081601}, we have
\begin{align*}
    &\TildeX\T \TildeX = \XX{i}{t}\T\myMa{-1 & & &\\& 1 & &\\& & \ddots &\\& & & 1}\myMa{-1 & & &\\& 1 & &\\& & \ddots &\\& & & 1}\XX{i}{t}=\XX{i}{t}\T \XX{i}{t},\\
    & \XX{i}{t}\T \deltawt{t-1} = [\XX{i}{t}]_{1,:},\  \TildeX\T \deltawt{t-1} = -[\XX{i}{t}]_{1,:}\text{ (here $[\cdot]_{1,:}$ denotes the first row of a matrix)}.
\end{align*}
Thus, we have
\begin{align}\label{eq.temp_081602}
    \TildeX(\TildeX\T\TildeX)^{-1}\TildeX\T \deltawt{t-1} = - \TildeX(\TildeX\T\TildeX)^{-1}\XX{i}{t}\T \deltawt{t-1}.
\end{align}
Therefore, we have
\begin{align*}
    &\XX{i}{t}(\XX{i}{t}\T\XX{i}{t})^{-1}\XX{i}{t}\T\deltawt{t-1} + \TildeX(\TildeX\T\TildeX)^{-1}\TildeX\T \deltawt{t-1}\\
    =& (\XX{i}{t}-\TildeX)(\XX{i}{t}\T \XX{i}{t})^{-1}\XX{i}{t}\T \deltawt{t-1}\ \text{ (by \cref{eq.temp_081602})}\\
    =& \myMa{2 & & &\\& 0 & &\\& & \ddots &\\& & & 0}\XX{i}{t}(\XX{i}{t}\T\XX{i}{t})^{-1} \XX{i}{t}\T \deltawt{t-1}\ \text{ (by \cref{eq.temp_081601})}\\
    =& \myMa{1 \\ 0 \\ \vdots \\ 0}\myMa{2 & 0 & \cdots & 0}\XX{i}{t}(\XX{i}{t}\T\XX{i}{t})^{-1} \XX{i}{t}\T \deltawt{t-1} \\
    =& 2\frac{\deltawt{t-1}}{C^2} \deltawt{t-1}\T \XX{i}{t}(\XX{i}{t}\T\XX{i}{t})^{-1} \XX{i}{t}\T \deltawt{t-1} \text{ (by \cref{eq.temp_081502})}\\
    =& 2\frac{\deltawt{t-1}}{C^2} \deltawt{t-1}\T\Pij{i}{t}\deltawt{t-1}\ \text{ (by \cref{eq.def_Pij})}\\
    =& 2\frac{\deltawt{t-1}}{C^2}\norm{\Pij{i}{t}\deltawt{t-1}}^2\ \text{ (since $\Pij{i}{t}\T\Pij{i}{t}=\Pij{i}{t}$ as $\Pij{i}{t}$ is an orthogonal projection)}. \numberthis \label{eq.temp_081503}
\end{align*}
Thus, we have
\begin{align*}
    &\E_{\XX{i}{t}} [\Pij{i}{t}\deltawt{t-1}]\\
    = &\int \XX{i}{t}(\XX{i}{t}\T\XX{i}{t})^{-1}\XX{i}{t}\T \deltawt{t-1} d\mu(\XX{i}{t})\\
    =& \frac{1}{2}\int \left(\XX{i}{t}(\XX{i}{t}\T\XX{i}{t})^{-1}\XX{i}{t}\T \deltawt{t-1} + \TildeX (\TildeX\T \TildeX)\TildeX\T \deltawt{t-1}\right)d\mu(\XX{i}{t})\ \text{ (by \cref{eq.temp_081501})}\\
    =& \int \frac{\deltawt{t-1}}{C^2} \norm{\Pij{i}{t}\deltawt{t-1}}^2 d \mu(\XX{i}{t})\\
    =&\frac{\deltawt{t-1}}{C^2}\E_{\XX{i}{t}} \norm{\Pij{i}{t}\deltawt{t-1}}^2\\
    =&\frac{\nij{i}{t}}{p} \deltawt{t-1}\ \text{ (by \cref{le.bias})}.
\end{align*}
The result of this lemma thus follows.
\end{proof}

\begin{lemma}\label[lemma]{le.model_error}
Let the noise in every test sample have zero mean and variance $\sigma^2$. For any learning result $\what$, the mean square test error must equal to $\norm{\what - \vwTrue}^2 + \sigma^2$. Therefore, the mean squared test error for noise-free test samples equals to the model error $L^{\text{model}}(\hat{\vw})=\norm{\hat{\vw}-\vwTrue}^2$.
\end{lemma}
\begin{proof}
Considering $(\bm{x}, y)$ as a randomly generated test sample by the ground truth $y=\bm{x}\T \vwTrue + \epsilon$, the mean squared error is equal to
\begin{align*}
    \E_{\bm{x},y} \norm{\bm{x}\T \what - y}
    =&\E_{\bm{x},\epsilon} \norm{\bm{x}\T \what - (\bm{x}\T \vwTrue + \epsilon)}^2 \\
    =& \E_{\bm{x},\epsilon} \norm{\bm{x}\T (\what - \vwTrue) + \epsilon}^2\\
    =& \E_{\bm{x}} \norm{\bm{x}\T (\what - \vwTrue)}^2 + \E_{\epsilon} \norm{\epsilon}^2\\
    & \text{ (since the noise $\epsilon$ has zero mean and is independent of other random variables)}\\
    =& \norm{\what - \vwTrue}^2 + \sigma^2\\
    & \text{ (notice that $\bm{x}$ follows standard Gaussian distribution and is independent of $\what$)}.
\end{align*}
\end{proof}





\section{Proof of Theorem~\ref{th.K1}}\label{app.proof_K1}

Calculating the gradient of the training loss defined at \cref{eq.train_loss}, we have
\begin{align*}
    \frac{\partial L(\what)}{\partial \what} =& \frac{\partial (\vy - \mX\T \what)}{\partial \what}\cdot \frac{\partial \frac{1}{2n}\norm{\vy - \mX\T \what}^2}{\partial (\vy - \mX\T\what)}\ \text{ (by the chain rule)}\\
    =& -\mX\cdot \frac{1}{n}(\vy - \mX\T \what)\\
    =&\frac{1}{n}(\mX\mX\T \what-\mX\vy).
\end{align*}

When $K=1$, with step size $\step>0$, we thus have
\begin{align*}
    \wwij{i}{t}=& \left(\iMatrix{p} - \frac{\step}{\nij{i}{t}}\XX{i}{t}\XX{i}{t}\T\right)\wwavg{t-1}+\frac{\step}{\nij{i}{t}}\XX{i}{t}\yy{i}{t}\\
    =& \left(\iMatrix{p} - \frac{\step}{\nij{i}{t}}\XX{i}{t}\XX{i}{t}\T\right)\wwavg{t-1}+\frac{\step}{\nij{i}{t}}\XX{i}{t}\left(\XX{i}{t}\T \wijTrue{i}{t} + \ee{i}{t}\right) \ \text{ (by \cref{eq.linear_model})}.
\end{align*}
Thus, we have
\begin{align*}
    \wwavg{t} =& \frac{1}{\sum_{i\in [m]}\nij{i}{t}} \sum_{i\in [m]}\nij{i}{t}\wwij{i}{t}\\
    =& \wwavg{t-1} + \mySum\step\left(- \XX{i}{t}\XX{i}{t}\T\wwavg{t-1} +  \XX{i}{t}\XX{i}{t}\T\wijTrue{i}{t} + \XX{i}{t}\ee{i}{t}\right).\numberthis \label{eq.temp_090101}
\end{align*}
By \cref{eq.def_delta,eq.define_gammait}, we have
\begin{align*}
    &\deltawwt{t} \\
    =& \deltawwt{t-1} + \mySum \step \left(\XX{i}{t}\XX{i}{t}\T(\gammait{i}{t}-\deltawwt{t-1})-\XX{i}{t}\ee{i}{t}\right)\\
    =&\mySum \Big(\underbrace{\left(\nij{i}{t}\iMatrix{p}-\step\XX{i}{t}\XX{i}{t}\T\right)\deltawwt{t-1}}_{\bm{q}_{1i}} + \underbrace{\step\XX{i}{t}\XX{i}{t}\T\gammait{i}{t}}_{\bm{q}_{2i}} - \underbrace{\step\XX{i}{t}\ee{i}{t}}_{\bm{q}_{3i}}\Big)\numberthis \label{eq.temp_090201}\\
    &\text{ (since $\deltawwt{t-1}=\mySum \nij{i}{t}\deltawwt{t-1}$)}.
\end{align*}
Considering the three types of terms $\bm{q}_{1i},\bm{q}_{2i},\bm{q}_{3i}$ defined in \cref{eq.temp_090201}, by \cref{as.Gaussian}, we have
\begin{equation}\label{eq.temp_090401}
    \begin{aligned}
        &\E_t \bm{q}_{1i} = \nij{i}{t}\left(1-\step\right)\deltawwt{t-1},\\
    &\E_t \bm{q}_{2i} = \step\nij{i}{t}\gammait{i}{t},\\
    &\E_t \bm{q}_{3i} = \bm{0}.
    \end{aligned}
\end{equation}
Notice that we use $\E$ to denote the expectation on all randomness and use $\E_{t}$ to denote the expectation on the randomness at the $t$-th round, i.e., on the randomness of $\XX{i}{t}$ and $\ee{i}{t}$ for all $i\in [m]$. 
By \cref{eq.temp_090401,eq.temp_090201}, we thus have
\begin{align*}
    \E_t \deltawwt{t} =& \mySum\left(\nij{i}{t}\left(1-\step\right)\deltawwt{t-1}+\step\nij{i}{t}\gammait{i}{t}\right).\numberthis \label{eq.temp_090402}
\end{align*}
Applying \cref{eq.temp_090402} recursively and recalling \cref{eq.def_gt}, we thus have
\begin{align}
    \E [\deltawwt{t}] = \bm{g}_t^{K=1}.\label{eq.temp_090403}
\end{align}

By \cref{as.Gaussian}, we know that $\ee{i}{t}$ is independent of $\XX{j}{t}$ for all $i,j\in [m]$ and $\E\ee{i}{t}=\bm{0}$. Thus, we have
\begin{align*}
    \E _t[\bm{q}_{1i}\T\bm{q}_{3j}]=\E_t[\bm{q}_{2i}\T\bm{q}_{3j}]=0.
\end{align*}
Thus, we have
\begin{align*}
    \E_t \norm{\deltawwt{t}}^2 = &\frac{1}{(\sum_{i\in [m]}\nij{i}{t})^2} \left(\sum_{i\in [m]} \left(\E_t \norm{\bm{q}_{1i}}^2 + \E_t \norm{\bm{q}_{2i}}^2 + \E_t \norm{\bm{q}_{3i}}^2 + 2\E_t [\bm{q}_{1i}\T\bm{q}_{2i}]\right)\right. \\
    & \left.+ \sum_{i\in [m]}\sum_{j\in [m]\setminus\{i\}}\left(\E_t [\bm{q}_{1i}\T\bm{q}_{1j}] + \E_t[\bm{q}_{1i}\T\bm{q}_{2j}]+\E_t[\bm{q}_{1j}\T\bm{q}_{2i}] + \E_t [\bm{q}_{2i}\T\bm{q}_{2j}]\right)\right)\\
    = &\frac{1}{(\sum_{i\in [m]}\nij{i}{t})^2} \left(\sum_{i\in [m]} \left(\E_t \norm{\bm{q}_{1i}}^2 + \E_t \norm{\bm{q}_{2i}}^2 + \E_t \norm{\bm{q}_{3i}}^2 + 2\E_t [\bm{q}_{1i}\T\bm{q}_{2i}]\right)\right. \\
    & \left.+ \sum_{i\in [m]}\sumNeq\left(\E_t [\bm{q}_{1i}\T\bm{q}_{1j}] + 2\E_t[\bm{q}_{1i}\T\bm{q}_{2j}] + \E_t [\bm{q}_{2i}\T\bm{q}_{2j}]\right)\right)\\
    &\text{ (since $\sum_{i\in [m]}\sumNeq \bm{q}_{1i}\T\bm{q}_{2j}+\bm{q}_{1j}\T\bm{q}_{2i}=2\sum_{i\in [m]}\sumNeq \bm{q}_{1i}\T\bm{q}_{2j}$)}.\numberthis \label{eq.temp_090204}
\end{align*}
By \cref{le.XXXX}, for any $i\in [m]$, we have
\begin{equation}\label{eq.temp_090203}
\begin{aligned}
    \E_t\norm{\bm{q}_{1i}}^2 =& \left(\nij{i}{t}^2 - 2\step\nij{i}{t}^2 + \step^2\nij{i}{t}(\nij{i}{t}+p+1)\right)\norm{\deltawwt{t-1}}^2\\
    =&\left(\left(1-\step\right)^2\nij{i}{t}^2+\step^2\nij{i}{t}(p+1)\right)\norm{\deltawwt{t-1}}^2,\\
    \E_t \norm{\bm{q}_{2i}}^2 =& \step^2\nij{i}{t}(\nij{i}{t}+p+1)\norm{\gammait{i}{t}}^2,\\
    \E_t \norm{\bm{q}_{3i}}^2 =& \step^2 p\nij{i}{t} \sigmaij{i}{t}^2,\\
    \E_t[\bm{q}_{1i}\T\bm{q}_{2i}] =&\left(\step\nij{i}{t}^2-\step^2\nij{i}{t}(\nij{i}{t}+p+1)\right)\deltawwt{t-1}\T\gammait{i}{t}.
\end{aligned}
\end{equation}
Similarly, by \cref{le.XXXX}, for any $i,j\in [m]$ where $i\neq j$, we have
\begin{equation}\label{eq.temp_090202}
\begin{aligned}
    \E [\bm{q}_{1i}\T\bm{q}_{1j}]=&\nij{i}{t}\nij{j}{t}\left(1-\step\right)\left(1-\stepj\right)\norm{\deltawwt{t-1}}^2,\\
    \E[\bm{q}_{1i}\T\bm{q}_{2j}]=&\left(\stepj\nij{i}{t}\nij{j}{t}-\step\stepj\nij{i}{t}\nij{j}{t}\right)\deltawwt{t-1}\T\gammait{j}{t}\\
    =&\nij{i}{t}\nij{j}{t}\stepj\left(1-\step\right)\deltawwt{t-1}\T\gammait{j}{t},\\
    \E[\bm{q}_{2i}\T\bm{q}_{2j}]=& \step\stepj\nij{i}{t}\nij{j}{t}\gammait{i}{t}\T\gammait{j}{t}.
\end{aligned}
\end{equation}
Plugging \cref{eq.temp_090203,eq.temp_090202} into \cref{eq.temp_090204}, we thus have
\begin{equation}\label{eq.temp_090404}
    \begin{aligned}
        &\E_t[\norm{\deltawwt{t}}^2] \\
    = &\frac{\norm{\deltawwt{t-1}}^2}{(\sum_{i\in [m]}\nij{i}{t})^2}\left(\sum_{i\in [m]}\left((1-\step)^2 \nij{i}{t}^2 + \step^2 \nij{i}{t}(p+1)\right) + \sum_{i\in [m]}\sum_{j\in [m]\setminus\{j\}}\nij{i}{t}\nij{j}{t}(1-\step)(1-\stepj)\right)\\
    &+\mySumSquare\step^2 \left(p\nij{i}{t} \sigmaij{i}{t}^2+\nij{i}{t}(\nij{i}{t}+p+1)\norm{\gammait{i}{t}}^2\right)\\
    &+2\mySumSquare\left(\step\nij{i}{t}^2-\step^2\nij{i}{t}(\nij{i}{t}+p+1)\right)\deltawwt{t-1}\T\gammait{i}{t}\\
    &+\mySumSquare\sumNeq\left(2\nij{i}{t}\nij{j}{t}\stepj\left(1-\step\right)\deltawwt{t-1}\T\gammait{j}{t} + \step\stepj\nij{i}{t}\nij{j}{t}\gammait{i}{t}\T\gammait{j}{t}\right).
    \end{aligned}
\end{equation}
Notice that
\begin{align*}
    &\left(\sum_{i\in [m]}\left((1-\step)^2 \nij{i}{t}^2 + \step^2 \nij{i}{t}(p+1)\right) + \sum_{i\in [m]}\sum_{j\in [m]\setminus\{j\}}\nij{i}{t}\nij{j}{t}(1-\step)(1-\stepj)\right) \\
    = &\frac{1}{(\sum_{i\in [m]}\nij{i}{t})^2} \left(\sum_{i\in [m]}\nij{i}{t}(1-\step)^2\right)^2 + \mySumSquare \step^2 \nij{i}{t}(p+1)\\
    =& H_t\ \text{(recalling \cref{eq.def_Ht})},
\end{align*}
and
\begin{align*}
    &\mySumSquare \step^2 \nij{i}{t}(\nij{i}{t} + p + 1)\norm{\gammait{i}{t}}^2 \\
    &+ \mySumSquare\sumNeq \step\stepj\nij{i}{t}\nij{j}{t}\gammait{i}{t}\T\gammait{j}{t}\\
    =&\frac{1}{(\sum_{i\in [m]}\nij{i}{t})^2} \norm{\sum_{i\in [m]}\step \nij{i}{t}\gammait{i}{t}}^2 + \mySumSquare \step^2 \nij{i}{t}(p+1) \norm{\gammait{i}{t}}^2,
\end{align*}
and
\begin{align*}
    &2\mySumSquare\left(\step\nij{i}{t}^2-\step^2\nij{i}{t}(\nij{i}{t}+p+1)\right)\deltawwt{t-1}\T\gammait{i}{t}\\
    &+\mySumSquare\sumNeq\left(2\nij{i}{t}\nij{j}{t}\stepj\left(1-\step\right)\deltawwt{t-1}\T\gammait{j}{t}\right)\\
    =& \frac{2}{(\sum_{i\in [m]}\nij{i}{t})^2}\left(\sum_{i\in [m]}\nij{i}{t}(1-\step)\right)\cdot \left(\sum_{i\in [m]}\nij{i}{t}\step \deltawwt{t-1}\T\gammait{i}{t}\right)\\
    &- \mySumSquareP{2\sum_{i\in [m]}\step^2 \nij{i}{t}(p+1)\deltawwt{t-1}\T \gammait{i}{t}}.
\end{align*}
Further, by \cref{eq.temp_090403} and recalling \cref{eq.def_Gt}, we thus can rewrite \cref{eq.temp_090404} as
\begin{align}\label{eq.temp_090405}
    \E \norm{\deltawwt{t}}^2 = H_t \E \norm{\deltawwt{t-1}}^2 + G_t.
\end{align}

Applying \cref{eq.temp_090405} recursively, we thus have \cref{eq.result_K1}.

\section{Proof of Theorem~\ref{th.multi_K}}\label{app.proof_multi_K}

Define
\begin{align*}
    &\bm{g}_l^{K<\infty} \defeq \F\left(l,\deltatGeneral{0},\seq{t}{\frac{\sum_{i\in [m]}\nij{i}{t}(1-\step)^K}{\sum_{i\in [m]}\nij{i}{t}}},\seq{t}{\frac{\sum_{i\in [m]}\nij{i}{t}\left(1-(1-\step)^K\right)\gammait{i}{t}}{\sum_{i\in [m]}\nij{i}{t}} }\right)\numberthis\label{eq.def_Fl}\\
    &\mathcal{A}_{(i),t}\defeq  (1-\step)^2 + \frac{\step^2(p+1)}{\nn},\numberthis \label{eq.def_Ait}
\end{align*}
\begin{align*}
    &\mathcal{B}_{(i),t,k} \defeq \frac{\step^2 p \sigmaij{i}{t}^2}{\nn}\\
    &+\left(\frac{\step^2}{\nn}(\nn + p + 1) + 2\step \left(1 - \frac{\step}{\nn}(\nn + p + 1)\right)\left(1-(1-\step)^{k-1}\right)\right)\norm{\gammait{i}{t}}^2\\
    &+2\left(\step - \frac{\step^2}{\nn}(\nn+p+1)\right)(1-\step)^{k-1} \gammait{i}{t}\T \bm{g}_{t-1}^{K<\infty},\numberthis \label{eq.def_Bit}
\end{align*}
\begin{align*}
    &\mathcal{J}_t \defeq \mySumSquarePP{\nij{i}{t}^2 \mathcal{A}_{(i),t}^K}+ \mySumSquarePP{\sumNeq \nij{i}{t}\nij{j}{t} (1-\step)^{K}(1-\stepj)^{K}},\numberthis \label{eq.def_Jt}\\
    &\mathcal{Q}_t \defeq   \mySumSquarePP {\nij{i}{t}^2 \sum_{k=1}^K \mathcal{B}_{(i),t,k} \mathcal{A}_{(i),t}^{K-k}}\\
    &+\mySumSquare\sumNeq \nij{i}{t}\nij{j}{t} \left(2(1-\step)^K (1-(1-\stepj)^K)\gammait{j}{t}\T\bm{g}_{t-1}^{K<\infty}\right.\\
    &\left.+(1-(1-\step)^K)(1-(1-\stepj)^K)\gammait{i}{t}\T\gammait{j}{t}\right).\numberthis \label{eq.def_Qt}
\end{align*}

In the following, we use $\E_k$ to denote the expectation with respect to the randomness in the $k$-th batch. 

We have
\begin{align*}
    \deltat{t} =& \vwTrue - \wavgg{t}\\
    =& \vwTrue - \mySum \nij{i}{t}\wijj{i}{t}\\
    =&\mySum \nij{i}{t}(\vwTrue - \wijj{i}{t})\ \text{ (since $\vwTrue = \mySum \nij{i}{t}\vwTrue$)}.
\end{align*}
Thus, we have
\begin{equation}\label{eq.temp_090901}
    \begin{aligned}
        \norm{\deltat{t}}^2 =& \mySumSquare \nij{i}{t}^2 \norm{\vwTrue - \wijj{i}{t}}^2 \\
    &+ \mySumSquare\sumNeq \nij{i}{t}\nij{j}{t}(\vwTrue - \wijj{i}{t})\T (\vwTrue - \wijj{j}{t}).
    \end{aligned}
\end{equation}
By \cref{as.Gaussian}, we know that at round $t$, different agents' data are independent with each other. Thus, we have
\begin{align*}
    \E_t (\vwTrue - \wijj{i}{t})\T (\vwTrue - \wijj{j}{t}) = \E_t (\vwTrue - \wijj{i}{t})\T \E_t (\vwTrue - \wijj{j}{t}).
\end{align*}
Thus, by \cref{eq.temp_090901}, to calculate $\E_t \norm{\deltat{t}}^2$, it remains to calculate $\E_t \norm{\vwTrue - \wijj{i}{t}}^2$ and $\E_t (\vwTrue - \wijj{i}{t})$ for all $i\in [m]$. To that end, we have
\begin{align*}
    \wk{k} = \left(\iMatrix{p}-\frac{\step}{\nn}\Xk\Xk\T\right)\wk{k-1} + \frac{\step}{\nn}\Xk(\Xk\T \wijTrue{i}{t}+\ek).
\end{align*}
We thus have
\begin{equation}\label{eq.temp_090903}
\begin{aligned}
\vwTrue - \wk{k} =&  \left(\iMatrix{p}-\frac{\step}{\nn}\Xk\Xk\T\right)(\vwTrue - \wk{k-1}) + \frac{\step}{\nn}\Xk\Xk\T(\vwTrue - \wijTrue{i}{t})\\
&+ \frac{\step}{\nn}\Xk \ek.
\end{aligned}
\end{equation}
By \cref{le.XXXX} and recalling \cref{eq.define_gammait}, we thus have
\begin{align}\label{eq.temp_090902}
    \E_{k}(\vwTrue - \wk{k}) = (1 -\step)(\vwTrue - \wk{k-1})+\step \gammait{i}{t}.
\end{align}
Applying \cref{eq.temp_090902} recursively and recalling that $\wk{0}=\deltat{t-1}$, we thus have
\begin{align}
    \E_{1,2,\cdots,k} (\vwTrue - \wk{k}) = (1-\step)^k  \deltat{t-1} + \left(1-(1-\step)^k\right) \gammait{i}{t}.\label{eq.temp_090904}
\end{align}
By letting $k=K$ in \cref{eq.temp_090904} and  $\wk{K}=\wijj{i}{t}$, we thus have
\begin{align}
    \E_{t} (\vwTrue - \wijj{i}{t}) = (1-\step)^K \deltat{t-1} + \left(1-(1-\step)^K\right) \gammait{i}{t}.\label{eq.temp_091002}
\end{align}
Plugging \cref{eq.temp_091002} into \cref{eq.temp_090901}, we thus have
\begin{align*}
    \E_t\norm{\deltat{t}}^2=& \mySumSquare \nij{i}{t}^2 \E_t\norm{\vwTrue - \wijj{i}{t}}^2 \\
    &+ \mySumSquare\sumNeq \nij{i}{t}\nij{j}{t}\E_t(\vwTrue - \wijj{i}{t})\T \E_t(\vwTrue - \wijj{j}{t})\numberthis \label{eq.temp_091301}\\
    =&\mySumSquare \nij{i}{t}^2 \E_t\norm{\vwTrue - \wijj{i}{t}}^2 \\
    &+ \mySumSquare\sumNeq \nij{i}{t}\nij{j}{t} \left((1-\step)^{K}(1-\stepj)^{K}\norm{\deltat{t-1}}^2\right.\\
    &\left.+(1-\step)^K (1-(1-\stepj)^K)\gammait{j}{t}\T\deltat{t-1}+(1-\stepj)^K (1-(1-\step)^K)\gammait{i}{t}\T\deltat{t-1}\right.\\
    &\left.+(1-(1-\step)^K)(1-(1-\stepj)^K)\gammait{i}{t}\T\gammait{j}{t}\right)\\
    =&\mySumSquare \nij{i}{t}^2 \E_t\norm{\vwTrue - \wijj{i}{t}}^2 \\
    &+ \mySumSquare\sumNeq \nij{i}{t}\nij{j}{t} \left((1-\step)^{K}(1-\stepj)^{K}\norm{\deltat{t-1}}^2\right.\\
    &\left.+2(1-\step)^K (1-(1-\stepj)^K)\gammait{j}{t}\T\deltat{t-1}\right.\\
    &\left.+(1-(1-\step)^K)(1-(1-\stepj)^K)\gammait{i}{t}\T\gammait{j}{t}\right).\numberthis \label{eq.temp_091302}
\end{align*}
Notice that in \cref{eq.temp_091301} we use $\E_t(\vwTrue - \wijj{i}{t})\T (\vwTrue - \wijj{j}{t})=\E_t(\vwTrue - \wijj{i}{t})\T \E_t(\vwTrue - \wijj{j}{t})$ for $i\neq j$, since $\wijj{i}{t}$ and $\wijj{j}{t}$ are independent with respect to the randomness during the local updates at round $t$.

By \cref{eq.FedAvg,eq.temp_091002}, we thus have
\begin{align}
    \E \deltat{t} = \frac{\sum_{i\in [m]}\nij{i}{t}(1-\step)^K}{\sum_{i\in [m]}\nij{i}{t}} \E \deltat{t-1} + \frac{\sum_{i\in [m]}\nij{i}{t}\left(1-(1-\step)^K\right)\gammait{i}{t}}{\sum_{i\in [m]}\nij{i}{t}}.\label{eq.temp_091003}
\end{align}
Applying \cref{eq.temp_091003} recursively and recalling \cref{eq.def_F_formula}, we thus have
\begin{align}
    \E [\deltat{l}] = \bm{g}_l^{K<\infty},\label{eq.temp_091101}
\end{align}
where $\bm{g}_l^{K<\infty}$ is defined in \cref{eq.def_Fl}.


By \cref{eq.temp_090903}, we have
\begin{align*}
    &\E_k \norm{\vwTrue - \wk{k}}^2 \\
    =& (\vwTrue - \wk{k-1})\T \left(\iMatrix{p} - 2\frac{\step}{\nn}\Xk\Xk\T + \frac{\step^2}{\nn^2}\Xk\Xk\T \Xk\Xk\T\right)(\vwTrue - \wk{k-1}) \\
    &+ \gammait{i}{t}\T \frac{\step^2}{\nn^2} \Xk\Xk\T\Xk\Xk \gammait{i}{t} + \ek\T \frac{\step^2}{\nn^2}\Xk\T \Xk \ek\\
    &+ 2 \frac{\step}{\nn} \gammait{i}{t}\T \Xk\Xk\T \left(\iMatrix{p}-\frac{\step}{\nn}\Xk\Xk\T\right) (\vwTrue - \wk{k-1})\\
    =& \left(1 - 2 \step + \frac{\step^2}{\nn}(\nn + p + 1)\right)\norm{\vwTrue - \wk{k-1}}^2 + \frac{\step^2}{\nn}(\nn + p + 1) \norm{\gammait{i}{t}}^2\\
    & + \step^2 \frac{p}{\nn} \sigmaij{i}{t}^2 + 2 \step \left(1 - \frac{\step}{\nn}(\nn + p + 1)\right)\gammait{i}{t}\T (\vwTrue - \wk{k-1})\ \text{ (by \cref{le.XXXX})}. \numberthis \label{eq.temp_090905}
\end{align*}
Plugging \cref{eq.temp_090904} into \cref{eq.temp_090905}, we have
\begin{align*}
    &\E_{1,2,\cdots,k} \norm{\vwTrue - \wk{k}}^2\\
    =& \left((1-\step)^2 + \frac{\step^2(p+1)}{\nn}\right) \E_{1,2,\cdots,k-1} \norm{\vwTrue - \wk{k-1}}^2 +\frac{\step^2}{\nn}(\nn + p + 1) \norm{\gammait{i}{t}}^2\\
    & + \step^2 \frac{p}{\nn} \sigmaij{i}{t}^2 + 2 \step \left(1 - \frac{\step}{\nn}(\nn + p + 1)\right) (1-\step)^{k-1} \gammait{i}{t}\T\deltat{t-1} \\
    &+  2\step \left(1 - \frac{\step}{\nn}(\nn + p + 1)\right)\left(1-(1-\step)^{k-1}\right) \norm{\gammait{i}{t}}^2\\
    =& \mathcal{A}_{(i),t}\E \norm{\vwTrue - \wk{k-1}}^2 + \mathcal{B}_{(i),t,k}',\numberthis \label{eq.temp_091201}
\end{align*}
where $\mathcal{A}_{(i),t}$ is defined in \cref{eq.def_Ait} and
\begin{align*}
    &\mathcal{B}_{(i),t,k}' \\
    \defeq &\frac{\step^2 p \sigmaij{i}{t}^2}{\nn}\\
    &+\left(\frac{\step^2}{\nn}(\nn + p + 1) + 2\step \left(1 - \frac{\step}{\nn}(\nn + p + 1)\right)\left(1-(1-\step)^{k-1}\right)\right)\norm{\gammait{i}{t}}^2\\
    &+2\left(\step - \frac{\step^2}{\nn}(\nn+p+1)\right)(1-\step)^{k-1} \gammait{i}{t}\T \deltat{t-1}.
\end{align*}
We also define $\mathcal{B}_{(i),t,k}$ by replacing $\deltat{t-1}$ in $\mathcal{B}_{(i),t,k}'$ with $\F_{t-1}$, i.e., \cref{eq.def_Bit}.

Applying \cref{eq.temp_091201} recursively over $k=1,2,\cdots,K$, we thus have
\begin{align}
    \E_t \norm{\vwTrue - \wijj{i}{t}}^2 =& \mathcal{A}_{(i),t}^K \norm{\deltat{t-1}}^2+ \sum_{k=1}^K \mathcal{B}_{(i),t,k} \mathcal{A}_{(i),t}^{K-k}.\label{eq.temp_091202}
\end{align}
Plugging \cref{eq.temp_091202,eq.temp_091101} into \cref{eq.temp_091302}, we thus have
\begin{align}\label{eq.temp_091203}
    \E \norm{\deltat{t}}^2 = & \mathcal{J}_t \E\norm{\deltat{t-1}}^2 +\mathcal{Q}_t,
\end{align}
where $\mathcal{J}_t$ is defined in \cref{eq.def_Jt} and $\mathcal{Q}_t$ is defined in \cref{eq.def_Qt}.

Applying \cref{eq.temp_091203} recursively, we thus have \cref{eq.multi_K}.

\section{Proof of Theorem~\ref{th.main}}\label{app.proof_main}

\begin{proof}
In the overparameterized situation, after each agent trains to converge, we have
\begin{align}\label{eq.temp_082301}
    \wij{i}{t}=&\XX{i}{t}\left(\XX{i}{t}\T\XX{i}{t}\right)^{-1} \left(\yy{i}{t}-\XX{i}{t}\T\wavg{t-1}\right)+\wavg{t-1}.
\end{align}
For any $i\in [m]$, we define $\Pij{i}{t}\in \mathds{R}^{p\times p}$ as
\begin{align}\label{eq.def_Pij}
    \Pij{i}{t}\defeq \XX{i}{t}\left(\XX{i}{t}\T\XX{i}{t}\right)^{-1} \XX{i}{t}\T.
\end{align}
(We know $\Pij{i}{t}$ is an orthogonal projection since $\Pij{i}{t}\Pij{i}{t}=\Pij{i}{t}$ and $\Pij{i}{t}\T = \Pij{i}{t}$.)
By \cref{eq.temp_082301,eq.linear_model,eq.def_Pij}, we thus have
\begin{align}
    \wij{i}{t} = \Pij{i}{t}\wijTrue{i}{t} + (\iMatrix{p} - \Pij{i}{t}) \wavg{t-1} + \XX{i}{t}\left(\XX{i}{t}\T\XX{i}{t}\right)^{-1} \ee{i}{t}.\label{eq.temp_082401}
\end{align}



We thus have
\begin{align*}
    &\deltawt{t} \\
    =&\vwTrue-\wavg{t}\ \text{ (by \cref{eq.def_delta})}\\
    =&\vwTrue - \mySum\nij{i}{t}\left(\Pij{i}{t}\wijTrue{i}{t} + (\iMatrix{p} - \Pij{i}{t}) \wavg{t-1} + \XX{i}{t}\left(\XX{i}{t}\T\XX{i}{t}\right)^{-1} \ee{i}{t}\right)\\
    &\text{ (by \cref{eq.temp_082401,eq.FedAvg})}\\
    =&\mySum\nij{i}{t}\left(\Pij{i}{t}(\vwTrue - \wijTrue{i}{t}) + (\iMatrix{p} - \Pij{i}{t}) (\vwTrue-\wavg{t-1}) - \XX{i}{t}\left(\XX{i}{t}\T\XX{i}{t}\right)^{-1} \ee{i}{t}\right)\\
    &\text{ (since $\vwTrue=\frac{\sum_{i\in [m]}\nij{i}{t}(\Pij{i}{t}+\iMatrix{p}-\Pij{i}{t})\vwTrue}{\sum_{i\in [m]}\nij{i}{t}}$)}\\
    =&\mySum\nij{i}{t}\left(\Pij{i}{t} \gammait{i}{t} + (\iMatrix{p}-\Pij{i}{t})\deltawt{t-1}- \XX{i}{t}\left(\XX{i}{t}\T\XX{i}{t}\right)^{-1} \ee{i}{t}\right)\\
    &\text{ (by \cref{eq.def_delta,eq.define_gammait})}.\numberthis \label{eq.delta_t_expression}
\end{align*}

For any $i,j\in [m]$, because $\ee{j}{t}$ is independent of $\deltawt{t-1}$ and $\XX{i}{t}$, and also because $\ee{j}{t}$ has zero mean (by \cref{as.Gaussian}), we have
\begin{align}
    &\E \left[\left(\Pij{i}{t}\gammait{i}{t}\right)\T \XX{j}{t}\left(\XX{j}{t}\T\XX{j}{t}\right)^{-1} \ee{j}{t}\right]\nonumber\\
    = &\E\left[((\iMatrix{p}-\Pij{i}{t})\deltawt{t-1})\T\XX{i}{t}\left(\XX{i}{t}\T\XX{i}{t}\right)^{-1} \ee{i}{t} \right] \nonumber\\
    =& 0,\label{eq.temp_081901}
\end{align}
and
\begin{align}
    &\E\left[\XX{i}{t}\left(\XX{i}{t}\T\XX{i}{t}\right)^{-1} \ee{i}{t} \right] = \bm{0}.\label{eq.temp_082902}
\end{align}
Since $\Pij{i}{t}(\iMatrix{p}-\Pij{i}{t})=\bm{0}$, we have
\begin{align}\label{eq.temp_092001}
    \left(\Pij{i}{t}\gammait{i}{t}\right)\T(\iMatrix{p}-\Pij{i}{t})\deltawt{t-1}=0.
\end{align}
Thus, by \cref{eq.delta_t_expression,eq.temp_081901,eq.temp_092001}, we have
\begin{align*}
    &\E_t \norm{\deltawt{t}}^2 \\
    =& \mySumSquareP{\sum_{i\in [m]} \nij{i}{t}^2\left(  \E_t\norm{(\iMatrix{p}-\Pij{i}{t})\deltawt{t-1}}^2 + \E_t\norm{\Pij{i}{t} \gammait{i}{t}}^2+\E_t\norm{\XX{i}{t}\left(\XX{i}{t}\T\XX{i}{t}\right)^{-1} \ee{i}{t}}^2\right) }\\
    &+ \mySumSquareP{1}\sum_{i\in [m]}\sumNeq \nij{i}{t}\nij{j}{t}\left(\gammait{j}{t}\T\Pij{j}{t}\Pij{i}{t}\gammait{i}{t}\right.\\
    &\left.+\deltawt{t-1}\T (\iMatrix{p}-\Pij{j}{t})(\iMatrix{p}-\Pij{i}{t})\deltawt{t-1}+2\gammait{j}{t}\T\Pij{j}{t}(\iMatrix{p}-\Pij{i}{t})\deltawt{t-1}\right).\numberthis \label{eq.temp_081902}
\end{align*}

For any $i\in [m]$, we have
\begin{align}
    &\E_t\norm{\Pij{i}{t}\gammait{i}{t}}^2 = \frac{\nij{i}{t}}{p}\norm{\gammait{i}{t}}^2\ \text{ (by \cref{le.bias})},\label{eq.temp_081903}\\
    &\E_t\norm{(\iMatrix{p}-\Pij{i}{t})\deltawt{t-1}}^2 = \left(1-\frac{\nij{i}{t}}{p}\right)\norm{\deltawt{t-1}}^2\ \text{ (by \cref{le.bias})},\label{eq.temp_092002}\\
    &\E_t\norm{\XX{i}{t}\left(\XX{i}{t}\T\XX{i}{t}\right)^{-1} \ee{i}{t}}^2 = \frac{\nij{i}{t}\sigma_i^2}{p-\nij{i}{t}-1}\ \text{ (by \cref{le.IW})}.\label{eq.temp_081904}
\end{align}
For any $i,j\in [m]$ where $i\neq j$, we have
\begin{align*}
    &\E_t \left[\deltawt{t-1}\T (\iMatrix{p}-\Pij{j}{t})(\iMatrix{p}-\Pij{i}{t})\deltawt{t-1}\right]\\
    =& \E_t \left[(\iMatrix{p}-\Pij{i}{t})\deltawt{t-1}\right]\T \E_t \left[(\iMatrix{p}-\Pij{j}{t})\deltawt{t-1}\right]\\
    &\text{ (since $\Pij{i}{t}$ and $\Pij{j}{t}$ are independent when $i\neq j$)}\\
    =& \left(1 - \frac{\nij{i}{t}}{p}\right)\left(1 - \frac{\nij{j}{t}}{p}\right)\norm{\deltawt{t-1}}^2 \text{ (by \cref{le.cross_term_exact_value})}.\numberthis \label{eq.temp_092101}
\end{align*}
Similarly, for $i\neq j$, we have
\begin{align*}
    \E_t \left[\gammait{j}{t}\T \Pij{j}{t}\Pij{i}{t}\gammait{i}{t}\right] = \frac{\nij{i}{t}\nij{j}{t}}{p^2} \gammait{j}{t}\T \gammait{i}{t}\ \text{ (by \cref{le.cross_term_exact_value})},\numberthis \label{eq.temp_092102}
\end{align*}
and
\begin{align*}
    \E_t \left[\gammait{j}{t}\T\Pij{j}{t}(\iMatrix{p}-\Pij{i}{t})\deltawt{t-1}\right] =\frac{\nij{j}{t}}{p}\left(1 - \frac{\nij{i}{t}}{p}\right)\gammait{j}{t}\T \deltawt{t-1} \ \text{ (by \cref{le.cross_term_exact_value})}.\numberthis \label{eq.temp_092103}
\end{align*}

Plugging \cref{eq.temp_081904,eq.temp_081903,eq.temp_092002,eq.temp_092101,eq.temp_092102,eq.temp_092103} into \cref{eq.temp_081902}, we thus have

\begin{align*}
    &\E_t \norm{\deltawt{t}}^2 \\
    =& \mySumSquareP{\sum_{i\in [m]} \nij{i}{t}^2\left(  \left(1-\frac{\nij{i}{t}}{p}\right)\norm{\deltawt{t-1}}^2 + \frac{\nij{i}{t}}{p}\norm{\gammait{i}{t}}^2+\frac{\nij{i}{t}\sigmaij{i}{t}^2}{p-\nij{i}{t}-1}\right) }\\
    &+ \mySumSquareP{1}\sum_{i\in [m]}\sumNeq \nij{i}{t}\nij{j}{t}\left(\frac{\nij{i}{t}\nij{j}{t}}{p^2} \gammait{j}{t}\T \gammait{i}{t}\right.\\
    &\left.+\left(1 - \frac{\nij{i}{t}}{p}\right)\left(1 - \frac{\nij{j}{t}}{p}\right)\norm{\deltawt{t-1}}^2+2\frac{\nij{j}{t}}{p}\left(1 - \frac{\nij{i}{t}}{p}\right)\gammait{j}{t}\T \deltawt{t-1}\right).\numberthis \label{eq.temp_082803}
\end{align*}


By \cref{eq.delta_t_expression}, we also have
\begin{align*}
    \E_t [\deltawt{t}] = \mySum \nij{i}{t} \left(\frac{\nij{i}{t}}{p}\gammait{i}{t} + \left(1 - \frac{\nij{i}{t}}{p}\right)\deltawt{t-1}\right).\numberthis \label{eq.temp_092105}
\end{align*}
Applying \cref{eq.temp_092105} recursively, we thus have
\begin{align}
    \E [\deltawt{l}] =\bm{g}_l^{K=\infty},\label{eq.E_Delta}
\end{align}
where $\bm{g}_l^{K=\infty}$ is defined in \cref{eq.def_F_inf}.

By \cref{eq.temp_082803,eq.E_Delta}, we thus have
\begin{align}
    \E \norm{\deltawt{t}}^2 = C_t \cdot  \E\norm{\deltawt{t-1}}^2 + D_t,\label{eq.temp_083101}
\end{align}
where $C_t$ denotes the coefficient of $\norm{\deltawt{t-1}}^2$ and $D_t$ denotes the remaining parts. The specific expressions of $C_t$ and $D_t$ are in \cref{eq.def_Ct,eq.def_Dt}.
Applying \cref{eq.temp_083101} recursively, we thus have \cref{eq.inf_K_main}.

\noindent\textbf{Underparameterized situation}

In the underparameterized situation, the convergence point of local steps in each round corresponds to the solution that minimizes the training loss, i.e.,
\begin{align*}
    \wij{i}{t} =& (\XX{i}{t}\XX{i}{t}\T)^{-1}\XX{i}{t}\yy{i}{t}\\
    =&(\XX{i}{t}\XX{i}{t}\T)^{-1}\XX{i}{t}(\XX{i}{t}\T \wijTrue{i}{t} + \ee{i}{t})\ \text{ (by \cref{eq.linear_model})}\\
    =&\wijTrue{i}{t} + (\XX{i}{t}\XX{i}{t}\T)^{-1}\XX{i}{t}\ee{i}{t}.
\end{align*}
Also recalling \cref{eq.def_delta,eq.define_gammait}, we thus have
\begin{align}
    \deltawt{t} = \mySum{\nij{i}{t}(\gammait{i}{t}-(\XX{i}{t}\XX{i}{t}\T)^{-1}\XX{i}{t}\ee{i}{t})}.\label{eq.temp_092108}
\end{align}
For any $i,j\in [m]$, because $\ee{j}{t}$ is independent of $\XX{i}{t}$ and $\ee{i}{t}$, and also because $\ee{j}{t}$ has zero mean (by \cref{as.Gaussian}), we have
\begin{align*}
    &\E\left[\gammait{j}{t}\T (\XX{i}{t}\XX{i}{t}\T)^{-1}\XX{i}{t}\ee{i}{t}\right] = 0\ \text{ for all $i,j\in [m]$},\\
    &\E\left[\left(\XX{j}{t}\XX{j}{t}\T)^{-1}\XX{j}{t}\ee{j}{t}\right)\T(\XX{i}{t}\XX{i}{t}\T)^{-1}\XX{i}{t}\ee{i}{t}\right]=0\ \text{ for all $i\neq j$}.
\end{align*}
Thus, by \cref{eq.temp_092108}, we have
\begin{align*}
    \E\norm{\deltawt{t}}^2 =& \mySumSquare \nij{i}{t}^2 \left(\norm{\gammait{i}{t}}^2 + \E\norm{(\XX{i}{t}\XX{i}{t}\T)^{-1}\XX{i}{t}\ee{i}{t}}^2\right)\\
    &+ \mySumSquare\sumNeq \nij{i}{t}\nij{j}{t}\gammait{i}{t}\T \gammait{j}{t}\\
    =& \norm{\mySumP{\nij{i}{t}\gammait{i}{t}}}^2 + \mySumSquarePP{ \frac{\nij{i}{t}^2 p\sigmaij{i}{t}^2}{\nij{i}{t}-p-1}} \text{ (by \cref{eq.temp_010303} in \cref{le.IW})}.
\end{align*}
We thus have proven \cref{eq.underparameterized}.



The result of this theorem thus follows.
\end{proof}

\section{A Table for Notations}\label{app.notation_table}
We provide a table of some important notations used in this paper.

\begin{table}[h!]
\centering
\begin{tabular}{||c | c | c||} 
 \hline
 {\bf symbol} & {\bf meaning}\\ 
 \hline\hline
 $\nij{i}{t}$ & number of training samples \\
 \hline
 $\nn$ & batch size\\
 \hline
 $p$ & number of parameters\\
 \hline
 $\sigmaij{i}{t}$ & noise level \\
 \hline
 $\XX{i}{t}$ & matrix for input of training samples\\
 \hline
 $\yy{i}{t}$ & vector for output of training samples\\
 \hline
 $\ee{i}{t}$ & vector for noise of training samples\\
 \hline
 $\wInit$ & the pre-trained parameters (initialization)\\
 \hline
 $\vwTrue$ & the learning target\\
 \hline
 $\wijTrue{i}{t}$ & the ground-truth of agent $i$ at round $t$\\
 \hline
 $\wwij{i}{t},\wijj{i}{t},\wij{i}{t}$ & the local learning result of agent $i$ at round $t$\\
 \hline
 $\wk{k}$ & learning result after $k$-th batch (for $K<\infty$ case)\\
 \hline
 $\wwavg{t},\wavgg{t},\wavg{t}$ &  the FedAvg result at round $t$\\
 \hline
 $\norm{\deltawwt{t}}^2,\norm{\deltat{t}}^2,\norm{\deltawt{t}}^2$ & model error\\
 \hline
 $\norm{\deltatGeneral{0}}^2$ & initial (pre-trained) model error \\
 \hline
 $\step$ & learning rate (step size)\\
 \hline
 $\gammait{i}{t}$ & measurement of heterogeneity\\
 \hline
\end{tabular}
\caption{Table for some notations.}
\label{table.notations}
\end{table}

\end{document}